\documentclass[11pt, a4paper, 
onecolumn, 
showabstract]{naverlabseurope-onecol-flex}

\setcitestyle{authoryear,round}

\usepackage[utf8]{inputenc}
\usepackage[T1]{fontenc}
\usepackage{amsmath,amssymb,amsthm,mathtools,bm,mathrsfs}
\usepackage{hyperref}
\usepackage{enumitem}
\usepackage{xspace}
\usepackage{booktabs}
\usepackage{graphicx}
\usepackage{xcolor}
\usepackage{alltt}
\usepackage{float}

\newtheorem{thm}{Theorem}[section]
\newtheorem{prop}[thm]{Proposition}
\newtheorem{cor}[thm]{Corollary}

\theoremstyle{remark}
\newtheorem{rem}[thm]{Remark}

\newenvironment{myhighlight}{\medskip\par\noindent\itshape}{\par\medskip}

\newcommand{\R}{\mathbb{R}}

\newcommand{\supp}{\mathrm{Supp}}
\newcommand{\KL}[2]{\mathrm{KL}\!\left(#1\,\|\,#2\right)}
\newcommand{\TVD}[2]{\mathrm{TVD}\!\left(#1,\,#2\right)}
\newcommand{\E}{\mathbb{E}}
\newcommand{\Y}{\mathcal{Y}}
\newcommand{\pb}[1][\beta]{p_{[#1]}}
\newcommand{\Zb}[1][\beta]{Z_{[#1]}}
\newcommand{\Jb}[1][\beta]{J_{[#1]}}

\newcommand{\pl}[1][\lambda]{p_{#1}}
\newcommand{\Zl}[1][\lambda]{Z_{#1}}

\newcommand{\muq}[1]{\mu_{#1}}

\newcommand{\Var}{\operatorname{Var}}

\newcommand{\pit}{\pi_\theta}

\renewcommand{\aa}{a}

\newcommand{\Aone}{A_1}
\newcommand{\Azero}{A_0}

\newcommand{\pp}{{p_*}}
\newcommand{\pihat}{\hat\pi_{\mathrm{FKL}}}

\newcommand{\PP}{\mathscr{P}}

\newcommand{\EE}{\mathcal{E}}

\newcommand{\blambda}{{\bm{\lambda}}}

\newcommand{\reinforce}{REINFORCE\xspace}

\DeclareMathOperator*{\argmax}{argmax}
\DeclareMathOperator*{\argmin}{argmin}

\authors{Marc Dymetman\textsuperscript{*}}
\affiliations{$^{*}$Scientific Consultant}
\website{\texttt{marc.dymetman@gmail.com}}
\websiteref{}

\title{Binary Rewards and Reinforcement Learning:\\ Fundamental Challenges}

\titlerunning{Binary Rewards and Reinforcement Learning: Fundamental Challenges}

\begin{abstract}
Reinforcement learning with verifiable rewards (RLVR) has become a
standard approach for improving reasoning in language models, yet
models trained with RLVR often suffer from diversity collapse: while
single-sample accuracy improves, multi-sample coverage degrades,
sometimes falling below the base model. We provide a structural account
of this phenomenon grounded in the properties of binary rewards.

Binary rewards create a fundamental degeneracy for policy gradient
methods: the set of distributions maximizing expected reward is
infinite, with no distinguished element. KL-control resolves this
degeneracy by selecting, in the limit $\beta\to 0$, the \emph{filtered
model} $\pp:=\aa(\cdot\mid\Y_1)$---the base model conditioned on
validity---which is the unique fully valid distribution closest to the
base model in KL divergence. This selection operates through a
nontrivial asymmetry: the tilted distribution
$\pb\propto\aa(y)\,e^{v(y)/\beta}$ converges to $\pp$ in forward KL
as $\beta\to 0$, yet $\pp$ cannot serve as a direct optimization target
because $\KL{q}{\pp}$ is infinite for any full-support policy $q$. We
develop explicit formulas relating the hyperparameter $\beta$ to the
more interpretable target validity rate $\mu$. Under model
misspecification---the typical practical regime---the pressure to
decrease $\beta$ drives the optimizer toward highly concentrated
distributions over a small number of valid outputs, collapsing toward
ever fewer as $\beta$ decreases, rather than toward the filtered model.
We illustrate this mechanism on a toy autoregressive experiment and
discuss how alternative divergences that target $\pp$ directly---as
pursued empirically by \citet{kruszewski_whatever_2026}---avoid this
failure mode by rewarding coverage of $\pp$'s support rather than
concentration on high-validity outputs.
\end{abstract}

\begin{document}
\maketitle

% ============================================================
\section{Introduction}
\label{sec:intro}
% ============================================================

Reinforcement learning with verifiable rewards (RLVR) has become a standard method for improving the reasoning capabilities of language models \citep{shao_deepseekmath_2024,guo_deepseek_2025}. In its simplest form, a binary verifier judges each model output as correct ($v(y)=1$) or incorrect ($v(y)=0$), and the policy is optimized to increase the proportion of correct outputs. The dominant algorithm, Group Relative Policy Optimization \citep[GRPO;][]{shao_deepseekmath_2024}, a variant of PPO \citep{schulman_proximal_2017} used to train DeepSeek-R1 \citep{guo_deepseek_2025}, optimizes precisely the KL-controlled objective studied in this note: a combination of expected reward and a KL penalty anchoring the policy to the base model, as popularized in the RLHF literature \citep{ouyang_training_2022}. Some recent work advocates dropping the KL penalty entirely \citep{yu_dapo_2025,liu_understanding_2025}, arguing that it is unnecessarily conservative for long chain-of-thought reasoning.

Growing evidence, however, shows that models trained with RLVR often suffer from a significant loss in diversity: while single-sample accuracy (pass@1) improves, multi-sample coverage (pass@$k$ for large~$k$) degrades, sometimes falling below the base model \citep{yue_does_2025,kruszewski_whatever_2026,li_choice_2025}. This suggests that RLVR does not so much create new reasoning capabilities as concentrate probability mass on a narrow subset of the base model's existing solutions---a phenomenon closely related to the mode collapse analyzed in the present note.

The simplest policy-gradient approach, \reinforce \citep{williams_simple_1992}, maximizes $\E_{\pit}[v(y)]$ by gradient ascent. 
With a binary verifier, however, this objective has a structural problem: the set of distributions achieving perfect validity ($\E_q v = 1$) is typically large, and the \reinforce objective is completely flat over this set. A Dirac mass on a single correct answer scores identically to a distribution that preserves the full diversity of the base model across all correct answers. Pure \reinforce has no mechanism to prefer one over the other.

KL-controlled optimization addresses this by adding a penalty $\beta\,\KL{\pit}{\aa}$ that anchors the policy to the base model. It is well known that for $\beta>0$, the optimal unconstrained distribution is a Gibbs distribution $\pb \propto \aa(y)\,e^{v(y)/\beta}$. We show that as $\beta\to 0$, this distribution converges to the \emph{filtered model} $\pp := \aa(\cdot\mid \Y_1)$---the base model conditioned on validity---which is the unique fully valid distribution closest to $\aa$ in KL divergence. In this sense, KL-control resolves the \reinforce degeneracy by implicitly selecting $\pp$.

But this resolution is more fragile than it appears. The convergence $\pb\to\pp$ holds in the forward KL direction, but the reverse KL, $\KL{\pb}{\pp}$, stays infinite for all $\beta>0$. One might therefore hope to substitute $\pp$ for $\pb$ as the optimization target, minimizing $\KL{\pit}{\pp}$ directly. But $\KL{\pit}{\pp}=\infty$ for any full-support policy $\pit$---which includes every standard autoregressive model---since $\pit$ has positive mass on invalid outputs, where $\pp$ has zero mass. The filtered model $\pp$ is thus structurally unreachable as a reverse-KL target, regardless of how small $\beta$ is made. More critically, we argue that under model misspecification---the practical regime where the parametric family $\Pi_\Theta$ does not contain $\pp$---the pressure to decrease $\beta$ tends to push the optimizer toward near-Dirac policies on a small number of valid outputs (mode collapse), rather than toward the filtered model.

The paper is organized as follows. Section~\ref{sec:setup} sets up the framework and identifies the binary-reward degeneracy. Section~\ref{sec:convergence} establishes the convergence of $\pb$ to $\pp$ and analyzes the ordering it induces among competing distributions. Sections~\ref{sec:binary-formulas} and~\ref{sec:beta-vs-mu} develop explicit formulas for the binary case and show that the opaque hyperparameter $\beta$ can be replaced by the more interpretable target validity rate $\mu$. Section~\ref{sec:misspecification} argues that misspecification disrupts the ideal picture. Section~\ref{sec:discussion} discusses the implications, including the question of whether alternative divergences could target $\pp$ directly.

Throughout, we advocate a shift in perspective: viewing KL-control not primarily as reward maximization with a penalty, but as distribution matching toward a target $\pb$. This geometric viewpoint clarifies both what KL-control achieves and where it fails.

The distributional matching perspective on RLVR, and the role of the
filtered model $\pp$ as the natural target, were developed in a line of
work including \citet{khalifa_distributional_2021},
\citet{korbak_reinforcement_2022}, \citet{kim_guaranteed_2025} and
\citet{kruszewski_whatever_2026}. The present note examines, in detail
and in a unified framework, the properties of KL-control specific to the
binary-reward setting: the precise sense in which $\pp$ emerges as a
limit (including a forward KL convergence result that appears to be
new), the structural reason why $\pp$ cannot serve as a direct
optimization target, and the mechanism by which misspecification drives
mode collapse. The technical tools---exponential families, I-projections,
Gibbs distributions---are standard; the contribution is the unified
picture they yield when brought to bear on binary rewards, providing a
structural account of phenomena that the empirical literature has been
observing piecemeal. This work was motivated in part by
\citet{kruszewski_whatever_2026}, of which the present author is a
co-author.
% ============================================================
\section{Setup and the Binary Reward Problem}
\label{sec:setup}
% ============================================================

\subsection{KL-controlled reinforcement learning}
\label{sec:kl-control}

We consider the following setup.
Let $\Y$ be a countable sample space (for concreteness, a space of finite token sequences) and $\PP$ the set of all probability distributions over~$\Y$.
A \emph{policy} $\pit$ is a parametrized distribution in $\PP$ from which one can both sample and evaluate probabilities; the parametric family is denoted~$\Pi_\Theta\subset\PP$.
We fix a \emph{base model} $\aa=\pi_{base}\in\Pi_\Theta$ and assume throughout that $\aa$ has full support over~$\Y$.
Finally, let $r:\Y\to\R$ be a real-valued reward.

The \emph{KL-control} objective, originating in linearly-solvable optimal control \citep{todorov_linearly_2007,kappen_linear_2005} 
adopted for language model fine-tuning by \citet{ziegler_fine-tuning_2019} and \citet{korbak_reinforcement_2022,korbak_rl_2022},
is, for a hyperparameter $\beta\ge 0$ and a distribution $q\in\PP$,
\begin{align}
    \Jb(q) \;:=\; \E_{y\sim q}\, r(y) \;-\; \beta\;\KL{q}{\aa}.
    \label{eq:J_beta_q}
\end{align}
The case $\beta=0$ recovers \reinforce: the objective $\Jb[0](q)=\E_q\,r$ depends only on the expected reward and makes no reference to the base model, 
opening the door to catastrophic forgetting.
For $\beta>0$ the KL penalty anchors the optimized policy to~$\aa$.

\subsection{The tilted distribution $\pb$}
\label{sec:tilted}

For $\beta>0$ and $r$ upper-bounded, define the \emph{tilted} (or Gibbs) distribution
\begin{align}
    \pb(y) \;:=\; \frac{1}{\Zb}\;\aa(y)\;e^{\frac{1}{\beta}\,r(y)},
    \qquad
    \Zb \;:=\; \sum_{y\in\Y}\aa(y)\;e^{\frac{1}{\beta}\,r(y)}.
    \label{eq:p_beta}
\end{align}
Because $\aa$ is full-support and $r$ is upper-bounded, $\Zb\in(0,\infty)$ and $\pb$ is a well-defined, full-support distribution.
Its role is captured by the following result (see Theorem 1 in \citealp{korbak_reinforcement_2022}).

\begin{prop}\label{thm:korbak}
Let $\beta>0$ and $r$ upper-bounded.
\begin{enumerate}[label=(\alph*), ref=\thethm(\alph*)]
    \item\label{thm:korbak:a}
        For any $q\in\PP$,\;
        $\Jb(q) = \beta\bigl[\log\Zb - \KL{q}{\pb}\bigr]$.
    \item\label{thm:korbak:b}
        $\pb$ is the unique maximizer of $\Jb(q)$ over $q\in\PP$.
    \item\label{thm:korbak:c}
        For any subset $\Pi\subset\PP$, maximizing $\Jb(q)$ over $\Pi$
        is equivalent to minimizing $\KL{q}{\pb}$ over $\Pi$.
\end{enumerate}
\end{prop}

\begin{proof}
For~\ref{thm:korbak:a}:
\begin{align*}
    \KL{q}{\pb}
    &= \E_q\log\frac{q(y)}{\pb(y)}
     = \log\Zb + \E_q\log\frac{q(y)}{\aa(y)\,e^{\frac{1}{\beta}r(y)}}\\
     &= \log\Zb - \tfrac{1}{\beta}\,\E_q\,r(y) + \KL{q}{\aa}\\
     &= \log\Zb - \tfrac{1}{\beta}\,\Jb(q).
\end{align*}
Parts \ref{thm:korbak:c} and \ref{thm:korbak:b} then follow:
for any $q,q'\in\Pi$,
$\KL{q}{\pb}<\KL{q'}{\pb}$ if and only if $\Jb(q)>\Jb(q')$;
and $\pb$ is the unique distribution satisfying $\KL{\pb}{\pb}=0$.
\end{proof}

Proposition~\ref{thm:korbak} converts the reward-maximization view of KL-control into a \emph{distribution-matching} view:
optimizing $\Jb$ over a model class $\Pi$ amounts to projecting $\Pi$ onto the target~$\pb$ in reverse KL.
This change of perspective will be central throughout the paper.

\begin{rem}[Role of KL-control for general rewards]\label{rem:reward-typology}
For a general reward $r$ whose maximum is attained at a single point~$y_*$,
the unique maximizer of $\E_q\,r$ over $\PP$ is the Dirac mass $\delta_{y_*}$, and the situation is unambiguous.
When $r$ does not attain its supremum, no maximizer of $\E_q\,r$ exists.
The critical intermediate case arises when the maximum is attained at \emph{multiple} points:
the set $\argmax_{q\in\PP}\E_q\,r$ is then an infinite family of distributions (all those supported on the maximizers), none of which is distinguished by the objective alone.
KL-control with $\beta>0$ resolves this degeneracy by selecting the unique maximizer~$\pb$ of~$\Jb$.
Binary rewards, our main focus, are the cleanest instance of this intermediate case.
\end{rem}

\subsection{Binary rewards and the REINFORCE degeneracy}
\label{sec:binary}

A \emph{binary reward} takes the form $r(y)=v(y)\in\{0,1\}$, where $v$ is a verifier.
Write
\[
\Y_1:=\{y\in\Y: v(y)=1\},
\qquad
\Y_0:=\Y\setminus\Y_1,
\]
and define $\Aone:=\aa(\Y_1)$, $\Azero:=\aa(\Y_0)=1-\Aone$.
We assume throughout that both $\Y_1$ and $\Y_0$ are nonempty
(so $0<\Aone<1$), and that $|\Y_1|\ge 2$
(the case of a single valid output being of limited interest).

The set of \emph{fully valid} distributions is
\begin{align}
    \PP_1 \;:=\; \{q\in\PP : \supp(q)\subseteq\Y_1\}
          \;=\; \{q\in\PP : \E_q\,v = 1\}.
    \label{eq:PP1}
\end{align}
This set coincides with $\argmax_{q\in\PP}\,\Jb[0](q)$:
every $q\in\PP_1$ achieves the maximum expected reward of~$1$.
In particular, $\PP_1$ contains every Dirac mass $\delta_{y^*}$ with $v(y^*)=1$, as well as every convex combination of such masses.

\begin{myhighlight}%
This is the \reinforce degeneracy for binary rewards:
the optimization landscape is completely flat over the large set~$\PP_1$.
A Dirac mass concentrated on a single valid output scores identically to a distribution that spreads its mass across all of~$\Y_1$ in proportion to the base model.
Pure \reinforce has no mechanism to prefer one over the other.
\end{myhighlight}

\subsection{The filtered model $\pp$}
\label{sec:pstar}

Among the elements of $\PP_1$, one distribution stands out on natural grounds.
The \emph{filtered model} is the base model conditioned on validity:
\begin{align}
    \pp(y) \;:=\; \aa(y \mid \Y_1) \;=\;
    \begin{cases}
        \aa(y)/\Aone & \text{if } v(y)=1,\\
        0            & \text{if } v(y)=0.
    \end{cases}
    \label{eq:pstar}
\end{align}
The filtered model preserves the relative probabilities assigned by $\aa$ to valid outputs:
for any $y,y'\in\Y_1$, $\pp(y')/\pp(y)=\aa(y')/\aa(y)$.
In this sense, it retains as much of the base model's structure as is compatible with perfect validity.

The following characterization makes this precise
(see Appendix~\ref{app:p*-is-I-proj-proof} for a proof):

\begin{prop}\label{thm:pstar-I-proj}
$\pp = \argmin_{q\in\PP_1}\KL{q}{\aa}$.
That is, $\pp$ is the unique distribution in $\PP_1$ that minimizes the divergence from $\aa$.
\end{prop}

In the language of information geometry, $\pp$ is the \emph{I-projection} of $\aa$ onto $\PP_1$ \citep{Csiszar1975IDivergenceGeometry}.

This gives $\pp$ a clear variational characterization:
it is the fully valid distribution closest to the base model.
But this characterization is not visible to pure \reinforce---the objective $\Jb[0]$ assigns $\pp$ the same value~$1$ as every other element of~$\PP_1$.
One of the main points of this paper is that KL-control \emph{does} single out $\pp$,
but only in a limiting sense, with subtleties that have practical consequences.

% ============================================================
\section{$\pp$ as a Limit of KL-Control}
\label{sec:convergence}
% ============================================================

\subsection{Convergence of $\pb$ to $\pp$}
\label{sec:convergence-thm}

The following theorem characterizes $\pp$ in terms of the limit behavior of $\pb$ for $\beta$ tending to $0$, distinguishing between different notions of distributional convergence (proof in Appendix~\ref{app:proof:convergence}).

\begin{thm}\label{thm:main}
Assume $\beta\to 0$. Then:
\begin{enumerate}[label=(\alph*), ref=\thethm(\alph*)]
\item\label{thm:main:a} $\forall y\in\Y,\; \pb(y)\to \pp(y)$.\hfill (pointwise)
\item\label{thm:main:b} $\TVD \pp \pb \to 0$.\hfill (total variation)
\item\label{thm:main:c} $\KL \pp \pb \to 0$.\hfill (forward KL)
\item\label{thm:main:d} $\KL \pb \pp = \infty$ for all $\beta>0$.\hfill (reverse KL)
\end{enumerate}
\end{thm}

Parts \ref{thm:main:a}--\ref{thm:main:b} were established by \citet{kruszewski_whatever_2026}; part~\ref{thm:main:d} (reverse KL infinity) is implicit in their Appendix~H, which analyzes the $\alpha\to 1$ limit of the $\alpha$-divergence $D_{f_\alpha}(\pi\|\pp)$ and shows that the leakage penalty diverges when $\pi$ has mass outside $\supp(\pp)$. We include all four parts here because the exponential-family framing yields a self-contained and unified treatment, and because part~\ref{thm:main:c} (forward KL convergence) appears to be new. Together, parts \ref{thm:main:a}--\ref{thm:main:c} represent progressively stronger forms of convergence: pointwise, total variation, and forward KL. The distribution $\pb$ approaches $\pp$ in all of these senses.

Part~\ref{thm:main:d} is the striking asymmetry. The reverse KL, $\KL{\pb}{\pp}$, is not merely large---it is \emph{infinite} for every finite $\beta$. This is because $\pb$ has full support over $\Y$ (since $\aa$ does), while $\pp$ is supported only on $\Y_1$: for any $y\in\Y_0$, $\pb(y)>0$ but $\pp(y)=0$, making the KL infinite.

This asymmetry has a concrete consequence. One might hope to replace $\pb$ by $\pp$ in the KL-control objective and minimize $\KL{q}{\pp}$ directly. But for any $q$ that puts any mass on $\Y_0$---which includes every full-support policy $\pit$---we have $\KL{q}{\pp}=\infty$. The reverse KL, $\KL{\cdot}{\pp}$, is simply not a usable objective for comparing distributions that are not already fully valid.

In contrast, $\pp$ can serve as the \emph{first} argument: $\KL{\pp}{q}$ does not suffer from the same structural obstruction, since $\pp$ is supported on $\Y_1$ and the sum $\sum_{y\in\Y_1}\pp(y)\log\frac{\pp(y)}{q(y)}$ is well-defined whenever $q$ has full support. In particular, autoregressive policies typically have full support, making $\KL{\pp}{\pit}$ a viable objective for targeting $\pp$ directly. We return to this observation in Section~\ref{sec:discussion}.

\subsection{Comparing candidates}
\label{sec:ordering}

Suppose that $\pi,\pi'$ are two arbitrary distributions in $\PP$, and 
that we want to compare their performances relative to the objective 
$\Jb$. We write $\muq{\pi}:=\E_\pi\,v$ for the \emph{validity} of 
$\pi$.\footnote{In the case of binary rewards, $\muq{\pi}=\E_\pi\,v$ 
measures the probability that a sample from $\pi$ is valid. We also 
call it the \emph{mean validity} of~$\pi$. When $\muq{\pi}=1$, we say 
$\pi$ is \emph{fully valid}.}
The following identities are immediate consequences of 
Proposition~\ref{thm:korbak:a}:
\begin{align}
    \Jb(\pi')-\Jb(\pi) 
    &= \beta \left[\KL{\pi}{\pb}-\KL{\pi'}{\pb}\right],
    \label{eq:Jb-compare-1}\\
    &= \left[\muq{\pi'} - \muq{\pi}\right] - \beta 
    \left[\KL{\pi'}{\aa} - \KL{\pi}{\aa}\right].
    \label{eq:Jb-compare-2}
\end{align}

From~\eqref{eq:Jb-compare-2}, the behavior as $\beta\to 0$ is 
transparent: the validity difference dominates, and among candidates 
with equal validity, the KL divergence to $\aa$ breaks ties. In 
particular:

\begin{figure}[t]
  \centering
  \includegraphics[width=0.75\linewidth]{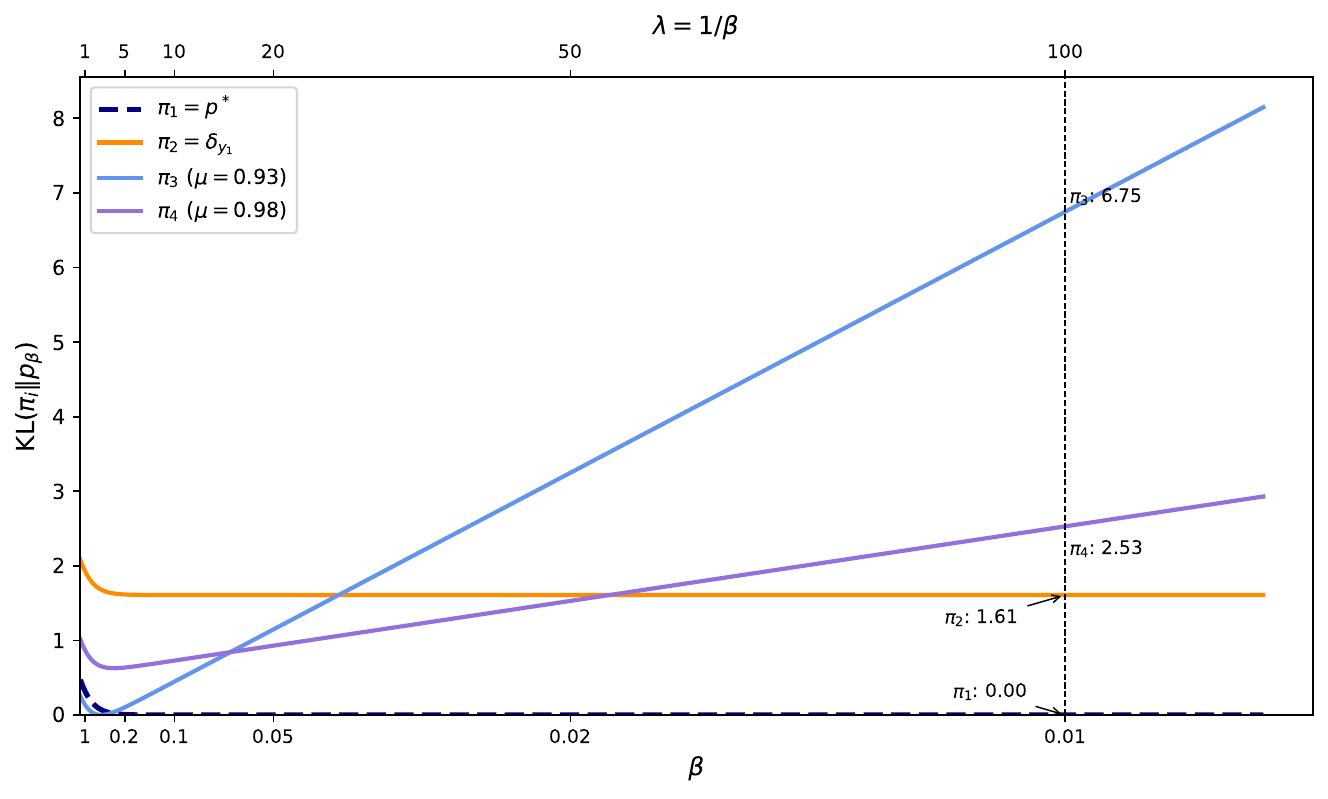}
\caption{Toy illustration ($|\Y|=5$, $|\Y_1|=3$) of the small-$\beta$ 
  ordering. We plot $\KL{\pi_i}{\pb}$ as a function of $\lambda=1/\beta$ 
  for four fixed candidates: $\pi_1=\pp$, $\pi_2=\delta_{y_1}$, 
  $\pi_3$ (validity $\muq{\pi_3}=0.93$, $\TVD{\pi_3}{\pp}=0.07$), and 
  $\pi_4$ (validity $\muq{\pi_4}=0.98$, $\TVD{\pi_4}{\pp}=0.54$).
  Although $\pi_3$ is much closer to $\pp$ in total variation than 
  $\pi_4$, it has lower validity. Consequently, for sufficiently large 
  $\lambda$, $\KL{\pi_4}{\pb}<\KL{\pi_3}{\pb}$: the optimizer prefers 
  $\pi_4$ over $\pi_3$, despite $\pi_4$ being far from $\pp$.
  This mechanism drives the mode-collapse argument of 
  Section~\ref{sec:misspecification}.}
  \label{fig:toy-kl-pi-vs-plambda}
\end{figure}

\begin{cor}[of Proposition~\ref{thm:pstar-I-proj}]\label{cor:pstar-wins}
For $\beta$ sufficiently small, $\pp$ is preferred over any $\pi\neq\pp$:
over any $\pi$ with $\muq{\pi}<1$ because $\muq{\pp}=1>\muq{\pi}$, 
and over any fully valid $\pi\neq\pp$ because $\KL{\pp}{\aa} < 
\KL{\pi}{\aa}$ by Proposition~\ref{thm:pstar-I-proj}.
\end{cor}

\begin{rem}\label{rem:pstar-limit-characterization}
The Corollary does \emph{not} say that $\pp$ maximizes $\Jb[0]$: 
indeed, $\Jb[0]$ assigns every fully valid distribution $\pi$ the same 
value of $1$, providing no basis for preferring $\pp$ over any other 
element of $\PP_1$. What the Corollary says is that the 
\emph{comparison} $\Jb(\pp)-\Jb(\pi)$, while tending to zero in 
absolute value as $\beta\to 0$, maintains a well-defined sign 
throughout: $\pp$ is ranked above any competitor for all sufficiently 
small $\beta>0$. Taking the limit before comparing destroys this 
information; comparing first and then taking the limit does not.
\end{rem}

\begin{rem}[KL view versus $\Jb$ view]\label{rem:kl-vs-jb}
Identity~\eqref{eq:Jb-compare-1} expresses the same comparison in 
terms of KL divergence to $\pb$: $\pi'$ is preferred over $\pi$ if 
and only if $\pi'$ is closer to $\pb$ in reverse KL. For two fully 
valid distributions $\pi,\pi'$ with $\muq{\pi}=\muq{\pi'}=1$, 
identity~\eqref{eq:Jb-compare-2} gives $\Jb(\pi')-\Jb(\pi) = 
-\beta[\KL{\pi'}{\aa}-\KL{\pi}{\aa}]$, which vanishes as $\beta\to 0$: 
the $\Jb$ objective becomes blind to the distinction between them. 
Yet from~\eqref{eq:Jb-compare-1}, the KL gap 
$\KL{\pi}{\pb}-\KL{\pi'}{\pb}$ remains nonzero, stabilizing at the 
constant $\KL{\pi}{\aa}-\KL{\pi'}{\aa}$ as $\beta\to 0$. The KL 
formulation maintains a nonvanishing gap that $\Jb$ alone cannot see.
\end{rem}

\citet[Appendix~H]{kruszewski_whatever_2026} develop a parallel 
analysis in the $\alpha$-divergence setting, decomposing 
$D_{f_\alpha}(\pi\|\pp)$ into a ``leakage penalty'' driven by 
$\pi(\Y_1)$ and a ``shape divergence'' measuring conditional fit. 
Their decomposition shows that for $\alpha$ close to~$1$, the leakage 
penalty dominates, so policies with higher validity mass are 
preferred---the same conclusion reached here via 
identity~\eqref{eq:Jb-compare-2}. The two viewpoints are 
complementary: theirs varies the divergence ($\alpha$) with the target 
fixed at $\pp$; ours varies $\beta$ with the divergence fixed at 
reverse KL.

\paragraph{Toy illustration.}
Figure~\ref{fig:toy-kl-pi-vs-plambda} illustrates the ordering 
mechanism on a toy instance with $|\Y|=5$ and $|\Y_1|=3$.
Let $\Y=\{y_1,\ldots,y_5\}$, $r(y_1)=r(y_2)=r(y_3)=1$, 
$r(y_4)=r(y_5)=0$, and
\[
\aa=(0.10,\;0.22,\;0.18,\;0.25,\;0.25),
\qquad
\pp=\aa(\cdot\mid\Y_1)=(0.20,\;0.44,\;0.36,\;0,\;0).
\]
We compare four candidates:
$\pi_1:=\pp$,\;
$\pi_2:=\delta_{y_1}=(1,0,0,0,0)$,\;
$\pi_3:=(1-\varepsilon)\pp+\varepsilon\,\nu_0$ with $\varepsilon=0.07$ 
and $\nu_0=(0,0,0,0.4,0.6)$,\;
$\pi_4:=(0.05,0.05,0.88,0.01,0.01)$.

We have $\muq{\pi_3}=0.93$, $\muq{\pi_4}=0.98$, 
$\TVD{\pi_3}{\pp}=0.07$, $\TVD{\pi_4}{\pp}=0.54$.
Thus $\pi_3$ is much closer to $\pp$ in total variation than $\pi_4$, 
but has lower validity. As identity~\eqref{eq:Jb-compare-2} predicts, 
for sufficiently small $\beta$ the optimizer prefers $\pi_4$ over 
$\pi_3$ despite $\pi_4$ being far from $\pp$: validity dominates.
Figure~\ref{fig:toy-kl-pi-vs-plambda} confirms this, showing 
$\KL{\pi_4}{\pb}<\KL{\pi_3}{\pb}$ for sufficiently large $\lambda=1/\beta$.

\subsection{Explicit formulas for binary rewards}
\label{sec:binary-formulas}

The tilted distribution $\pb$ of~\eqref{eq:p_beta} belongs to the
\emph{exponential family} generated by $\aa$ and the reward $r$
(a connection we develop further in Section~\ref{sec:general-info-geom}):
\begin{align}
    \pl(y) \;:=\; \frac{\aa(y)\,e^{\lambda\,r(y)}}{\Zl},
    \qquad
    \Zl \;:=\; \sum_{y\in\Y}\aa(y)\,e^{\lambda\,r(y)},
    \qquad
    A(\lambda) \;:=\; \log \Zl,
    \label{eq:1-dim-exp-fam}
\end{align}
where $\lambda\in\R$ is the \emph{natural parameter}. Under the
identification $\lambda=1/\beta$, we have $\pl=\pb$.

For $r(y)=v(y)\in\{0,1\}$, this family takes a particularly simple form: because the reward takes only two values, all the relevant quantities ($\Zl$, $\mu(\lambda)$, $\lambda(\mu)$, $\kappa(\mu)$) admit closed elementary expressions in terms of $\Aone$ and $\Azero$. This is special to the binary case; for general bounded rewards the same quantities exist and retain the same geometric meaning, but no longer in closed form (see Appendix~\ref{app:general-info-geom}).
Recall that $\Aone:=\aa(\Y_1)\in(0,1)$ and $\Azero:=1-\Aone=\aa(\Y_0)$.
Then
\[
\Zl=\Azero+\Aone\,e^\lambda,
\qquad
A(\lambda)=\log\bigl(\Azero+\Aone\,e^\lambda\bigr).
\]
The three key functions $\mu$, $\lambda$, $\kappa$ are defined as follows.

\begin{description}

\item[$\mu(\lambda)$:] The \emph{moment map}, or \emph{target validity},
is the expected reward under $\pl$:
\[
\mu(\lambda) \;=\; \E_{\pl}[r] \;=\; \frac{\Aone\,e^\lambda}{\Azero+\Aone\,e^\lambda}.
\]
This is a strictly increasing bijection from $\R$ to $(0,1)$: as $\lambda$
increases, the numerator grows faster than the denominator, so
$\mu(\lambda)$ increases; and $\mu(\lambda)\to 0$ as $\lambda\to-\infty$,
$\mu(\lambda)\to 1$ as $\lambda\to+\infty$. At $\lambda=0$ (the base
model), $\mu(0)=\Aone$.

\item[$\lambda(\mu)$:] Since $\mu(\lambda)$ is a strictly increasing
bijection, it has an inverse $\lambda(\mu)$, obtained by solving
$\mu=\Aone e^\lambda/(\Azero+\Aone e^\lambda)$ for $\lambda$:
\[
\lambda(\mu) \;=\; \log\frac{\mu}{1-\mu}+\log\frac{\Azero}{\Aone}.
\]
This is a strictly increasing bijection from $(0,1)$ to $\R$,
with $\lambda(\Aone)=0$.

\item[$\kappa(\mu)$:] The \emph{divergence cost} $\kappa(\mu):=
\KL{\pl[\lambda(\mu)]}{\aa}$ is the KL price paid by the exponential
family to achieve validity $\mu$:
\begin{align}
    \kappa(\mu) \;=\; \mu\log\frac{\mu}{\Aone}+(1-\mu)\log\frac{1-\mu}{\Azero}.
    \label{eq:kappa-binary}
\end{align}
% This is the KL divergence from $\mathrm{Bernoulli}(\mu)$ to
% $\mathrm{Bernoulli}(\Aone)$. It is nonnegative, strictly convex on
% $(0,1)$, has its unique minimum $0$ at $\mu=\Aone$ (the base model's
% validity rate), and is strictly increasing on $[\Aone,1)$ with
% $\kappa(1)=\KL{\pp}{\aa}=-\log\Aone$.
This is the KL divergence from $\mathrm{Bernoulli}(\mu)$ to
$\mathrm{Bernoulli}(\Aone)$. It is nonnegative, strictly convex on
$(0,1)$, has its unique minimum $0$ at $\mu=\Aone$ (the base model's
validity rate), and is strictly increasing on $[\Aone,1)$. Using the
standard convention $0\log 0:=0$, it extends continuously to $\mu=1$,
where $\kappa(1)=\KL{\pp}{\aa}=-\log\Aone$.

\end{description}

The three quantities $\lambda$, $\mu$, $\kappa$ are in bijection over
the regime of interest: fixing any one determines the other two.
In particular, the hyperparameter $\beta=1/\lambda$ determines the
target validity~$\mu$ and the KL cost~$\kappa$, and vice versa.
This reflects a fundamental trade-off: $\kappa(\mu)$ is strictly
increasing on $[\Aone,1)$, so achieving higher validity always requires
paying a larger KL cost relative to the base model. As $\mu\to 1$
(equivalently $\lambda\to+\infty$), the cost approaches its supremum
$\kappa(1)=\KL{\pp}{\aa}=-\log\Aone$, which is the KL divergence of
the filtered model from the base model---a finite but
unattained ceiling within the open exponential family
$\{\pl\}_{\lambda\in\R}$, attained only in the limit $\lambda\to+\infty$.

For each $\mu\in(0,1)$, define the \emph{moment slice}
$\mathcal{M}_\mu:=\{q\in\PP:\E_q[r]=\mu\}$ as the set of all
distributions with expected reward~$\mu$. The exponential family
$\EE:=\{\pl\}_{\lambda\in\R}$ intersects each slice $\mathcal{M}_\mu$
at exactly the point $\pl[\lambda(\mu)]$, since $\mu(\lambda)$ is a
bijection. As established in Theorem~\ref{thm:main}, the family
interpolates from $p_{-*}:=\aa(\cdot\mid\Y_0)$ to
$\pp:=\aa(\cdot\mid\Y_1)$ as $\lambda$ runs from $-\infty$ to
$+\infty$, with $\mu(\lambda)$ running from~$0$ to~$1$.
Figure~\ref{fig:binary-info-geom} illustrates this geometry.

\begin{figure}[t]
    \centering
    \includegraphics[width=0.85\linewidth]{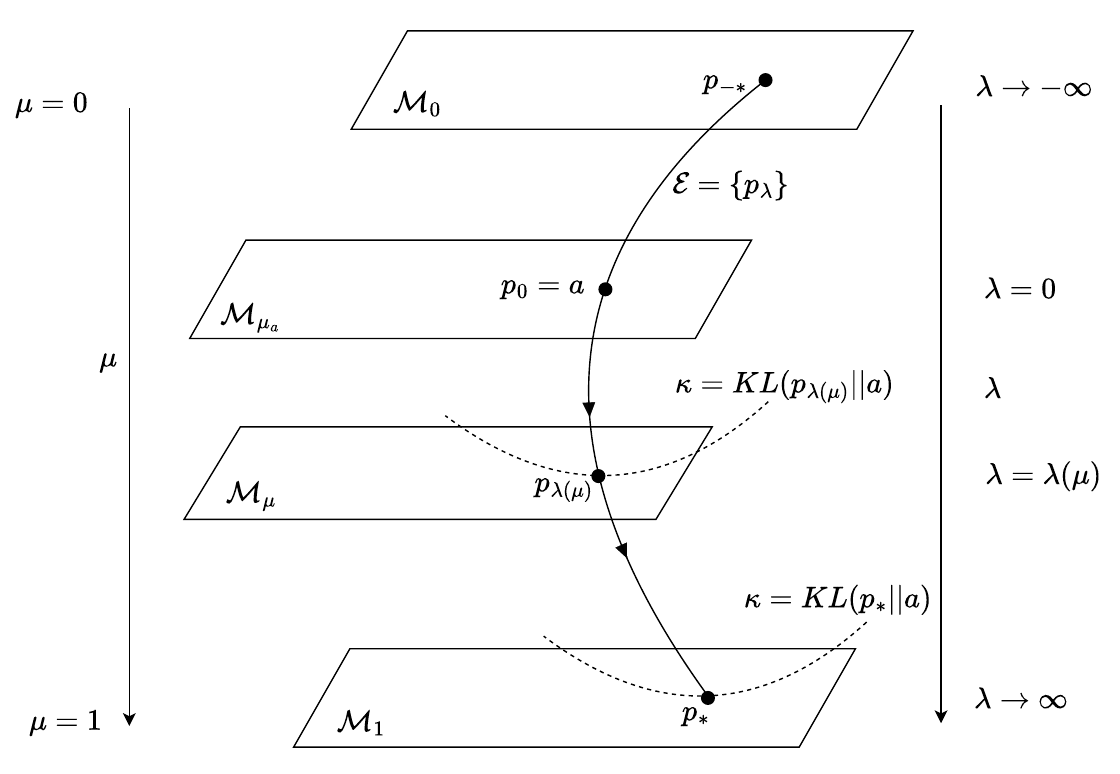}
\caption{Information geometry of KL-control for binary rewards.
    The \emph{moment slices} $\mathcal{M}_\mu$ are indexed by the
    validity level $\mu\in[0,1]$ and stacked vertically with $\mu$
    increasing downward: $\mathcal{M}_0$ (top, containing $p_{-*}$),
    a representative interior slice $\mathcal{M}_{\mu_a}$ (containing
    $\aa=p_0$, the base model at $\lambda=0$), a generic slice
    $\mathcal{M}_\mu$, and $\mathcal{M}_1$ (bottom, equal to $\PP_1$
    and containing $\pp$).
    The exponential family $\mathcal{E}=\{p_\lambda\}$ (curve with
    arrows) threads through all slices, with each $\pl[\lambda(\mu)]$
    the I-projection of $\aa$ onto $\mathcal{M}_\mu$ (see
    Subsection~\ref{sec:general-info-geom} and
    Appendix~\ref{app:general-info-geom} for this notion and
    Fig.~\ref{fig:iprojs-general} for the general bounded-reward
    skeleton without attained bounds); the arrows indicate this
    projection relationship.
    The dashed curves show two KL sublevel sets: one tangent to
    $\mathcal{M}_\mu$ at $\pl[\lambda(\mu)]$ with cost
    $\kappa=\KL{\pl[\lambda(\mu)]}{\aa}$, and one tangent to
    $\mathcal{M}_1$ at $\pp$ with cost $\kappa=\KL{\pp}{\aa}$.}
    \label{fig:binary-info-geom}
\end{figure}

\subsection{The interpretability of $\beta$ versus $\mu$}
\label{sec:beta-vs-mu}

The case of binary rewards highlights a practical issue with KL-control as commonly used. The hyperparameter $\beta$ (or equivalently $\lambda=1/\beta$) has no transparent meaning: it controls the balance between reward and KL penalty, but its numerical value does not directly convey the properties of the resulting distribution.

In contrast, the validity target $\mu$ has a clear invariant meaning: it specifies the fraction of samples that should be valid.
Setting $\mu=0.9$, for instance, means ``$90\%$ of the model's outputs should satisfy the verifier.'' This is a specification that is meaningful independently of the base model.

The relation between $\beta$ and $\mu$ depends strongly on $\Aone=\aa(\Y_1)$, the base model's current validity rate.
From the table above, $\beta = 1/\lambda(\mu) = 1/\!\left(\log\frac{\mu}{1-\mu}+\log\frac{\Azero}{\Aone}\right)$.
To target $\mu=0.9$, one needs $\lambda(0.9)=\log 9 + \log\frac{\Azero}{\Aone}$, which equals approximately $4.4$ when $\Aone=0.1$ (weak base model), $2.2$ when $\Aone=0.5$ (moderate), and $0$ when $\Aone=0.9$ (strong). The corresponding $\beta$ values are $0.23$, $0.45$, and $\infty$, respectively: the same target validity corresponds to wildly different $\beta$'s depending on the base model.

\begin{myhighlight}%
Fixing $\beta$ across experiments with different base models conflates two effects---the target validity level and the base model's strength---and is therefore difficult to interpret. Fixing $\mu$ separates them: it specifies what we want, independently of where we start.
\end{myhighlight}

\noindent A practical caveat: exploiting $\mu$ directly in the optimization requires knowing $\Aone$, which may need to be estimated. In contrast, $\beta$ can be plugged directly into the objective~\eqref{eq:J_beta_q}. Nevertheless, the conceptual point stands---$\mu$ is the more interpretable quantity---and adaptive schemes that estimate $\Aone$ on the fly are straightforward.\footnote{For instance, one can estimate $\Aone$ by evaluating the verifier on a batch of base-model samples, and then set $\beta=1/\lambda(\mu)$ for the desired target $\mu$.}

\subsection{Connection to Exponential Families and Information Geometry}
\label{rem:general-info-geom}
\label{sec:general-info-geom}

The binary formulas of Section~\ref{sec:binary-formulas} are a
specialization of the theory of \emph{exponential families} \citep{brown_fundamentals_1986, wainwright_graphical_2008}, and it is natural to work directly
in that language.
Three considerations motivate
the switch from $\beta$ to the natural parameter $\lambda=1/\beta$.
First, $\lambda$ ranges over all of $\R$: negative values correspond to
distributions tilted \emph{away} from validity, and $\lambda=0$ recovers
the base model $\aa$ itself---a structural fact that is obscured when
$\beta$ is confined to $(0,\infty)$. Second, the exponential family
framing connects to a rich body of theory (I-projections, Legendre
duality, the geometry of moment slices) that we can exploit directly.
Third, and looking beyond the one-dimensional setting of this paper,
$\lambda$ generalizes naturally to a vector of natural parameters
$\blambda\in\R^k$ for multi-reward objectives, whereas $\beta$ has no
natural vector analogue.

A useful algebraic identity, whose consequences for exponential family
theory are developed systematically in \citet{dymetman_exponential_2026},
is the following, where $A(\lambda):=\log\Zl$ is the \emph{log-partition
function} (see Appendix~\ref{app:kl-identity-proof} for a proof):
\begin{equation}\label{eq:KL-identity}
    \KL{q}{\pl[{\lambda_2}]} - \KL{q}{\pl[{\lambda_1}]}
    = A(\lambda_2) - A(\lambda_1) + \muq{q} \cdot (\lambda_1 - \lambda_2).
\end{equation}
For instance, setting $\lambda_1=0$ (so $\pl[\lambda_1]=\aa$) and
rearranging recovers the identity of
Proposition~\ref{thm:korbak:a}: $\Jb(q) = \beta[\log\Zb -
\KL{q}{\pb}]$.

The geometric picture that emerges is as follows. For arbitrary bounded
rewards, the exponential family curve $\EE=\{\pl\}$ threads through a
family of \emph{moment slices} $\mathcal{M}_\mu:=\{q\in\PP:\E_q[r]=\mu\}$,
intersecting each slice $\mathcal{M}_\mu$ at the unique point
$\pl[\lambda(\mu)]$. This intersection point is the \emph{I-projection}
of $\aa$ onto $\mathcal{M}_\mu$---the distribution in $\mathcal{M}_\mu$
closest to $\aa$ in KL divergence. The divergence cost
$\kappa(\mu)=\KL{\pl[\lambda(\mu)]}{\aa}$ is the Legendre--Fenchel dual
of the log-partition function $A$, and the KL sublevel set of cost
$\kappa(\mu)$ is tangent to $\mathcal{M}_\mu$ at $\pl[\lambda(\mu)]$:
any distribution with expected reward exceeding $\mu$ must pay a KL
cost exceeding $\kappa(\mu)$. This general treatment, which applies to
arbitrary bounded rewards, is developed in
Appendix~\ref{app:general-info-geom}.

% ============================================================
\section{Misspecification and Mode Collapse}
\label{sec:misspecification}
% ============================================================

The previous sections have described an ideal picture: in the space $\PP$ of all distributions, KL-control selects $\pp$ in the $\beta\to 0$ limit, and the trade-off between validity and divergence is cleanly characterized. In practice, however, the optimization is over a parametric family $\Pi_\Theta$ of autoregressive policies, and this changes the picture fundamentally.

\subsection{The misspecification problem}
\label{sec:misspec-problem}

The parametric family $\Pi_\Theta$ typically includes the base model $\aa$ but cannot be assumed to contain $\pp$, or even any other member of the exponential family $\EE=\{\pl\}$ beyond $\aa=\pl[0]$.

\begin{quote}
There exist standard autoregressive models $\aa$ and low-complexity binary verifiers $v$ such that no autoregressive model $\pit$ can match the filtered model $\pp(y)\propto \aa(y)\,v(y)$; see \citet{lin_limitations_2021} and Appendix~A of \citet{kim_guaranteed_2025}. Even simple verifiers, such as checking the presence of a specific word, are suspected to produce $\pp$ outside $\Pi_\Theta$ \citep{khalifa_distributional_2021, zhang_tractable_2023}.
\end{quote}

A parametric family that does not contain the target distribution is
called \emph{misspecified}. Under misspecification, the optimization
$\inf_{\pit\in\Pi_\Theta}\KL{\pit}{\pb}$ may not reach its
unconstrained optimum $\pb$, and the best achievable policy within
$\Pi_\Theta$ (assuming the infimum is attained) may bear little
resemblance to either $\pb$ or $\pp$.

\subsection{Two compounding gaps}
\label{sec:two-gaps}

The practical optimization faces two distinct difficulties:

\paragraph{The misspecification gap.}
Even if the optimizer could perfectly solve the projection problem, the best achievable policy $\pi_\theta^*:=\argmin_{\pit\in\Pi_\Theta}\KL{\pit}{\pb}$ may be far from $\pp$ (or from $\pb$). The structure of $\Pi_\Theta$ constrains what distributions are reachable, and the KL-optimal projection onto $\Pi_\Theta$ may land on a qualitatively different distribution than the target.

\paragraph{The optimization gap.}
Standard stochastic-gradient methods for minimizing $\KL{\pit}{\pb}$ (equivalently, maximizing $\Jb(\pit)$) follow complex dynamics that do not guarantee convergence to the global optimum $\pi_\theta^*$ within $\Pi_\Theta$. Local optima, saddle points, and the high-dimensional landscape of autoregressive models all contribute to this gap.

Both gaps are present simultaneously in practice. The analysis below focuses on the misspecification gap, which is the more fundamental of the two and which interacts specifically with the structure of binary rewards.

\subsection{Why small $\beta$ drives mode collapse}
\label{sec:mode-collapse}
The key insight comes from identity~\eqref{eq:Jb-compare-2}, which we
rewrite in terms of $\lambda=1/\beta$:
\begin{align}
    \KL{\pi}{\pl} - \KL{\pi'}{\pl}
    = \bigl[\KL{\pi}{\aa}-\KL{\pi'}{\aa}\bigr] - \lambda\bigl[\muq{\pi}-\muq{\pi'}\bigr]
    \label{eq:KL-difference-candidates}
\end{align}
(equivalently, a direct consequence of identity~\eqref{eq:KL-identity}).
The first bracket is a fixed quantity (independent of $\lambda$), while
the second is amplified by $\lambda$. For large $\lambda$ (small
$\beta$), any validity advantage $\muq{\pi'}>\muq{\pi}$ eventually
dominates, regardless of how the two policies compare in proximity to
the base model.

This has a specific consequence for the structure of autoregressive model families:

\paragraph{Near-Dirac policies are easy; $\pp$ is hard.}
An autoregressive model can easily concentrate its probability mass on a single valid output $y^*=[y^*_1,\ldots,y^*_T]\in\Y_1$: it suffices to learn, at each position~$t$, to assign high probability to the single token $y^*_t$ given the prefix $y^*_{<t}$. This is memorization of a single path through the token tree---a task so simple that even a deterministic $n$-gram model can accomplish it. Such a policy achieves $\muq{\pi}\approx 1$. In contrast, the filtered model $\pp=\aa(\cdot\mid\Y_1)$ requires spreading mass across all of $\Y_1$ in proportion to $\aa$---a global distributional constraint that demands coordinating the conditional distributions at every branching point of the token tree. This is a much harder task that autoregressive models are known to have difficulty satisfying, and in some cases provably cannot \citep{lin_limitations_2021,kim_guaranteed_2025}.

\paragraph{The mode-collapse mechanism.}
Under small $\beta$ (large $\lambda$), the optimizer must minimize $\KL{\pit}{\pl}$. Equation~\eqref{eq:KL-difference-candidates} shows that it is overwhelmingly rewarded for increasing $\muq{\pit}$. The policies in $\Pi_\Theta$ with the highest validity tend to be highly concentrated distributions---in the limit, near-Dirac masses on a small number of valid outputs. Therefore, the optimization path tends toward such policies---mode collapse---rather than toward $\pp$, which has $\muq{\pp}=1$ but is unreachable.

\begin{figure}[t]
    \centering
    \includegraphics[width=0.85\linewidth]{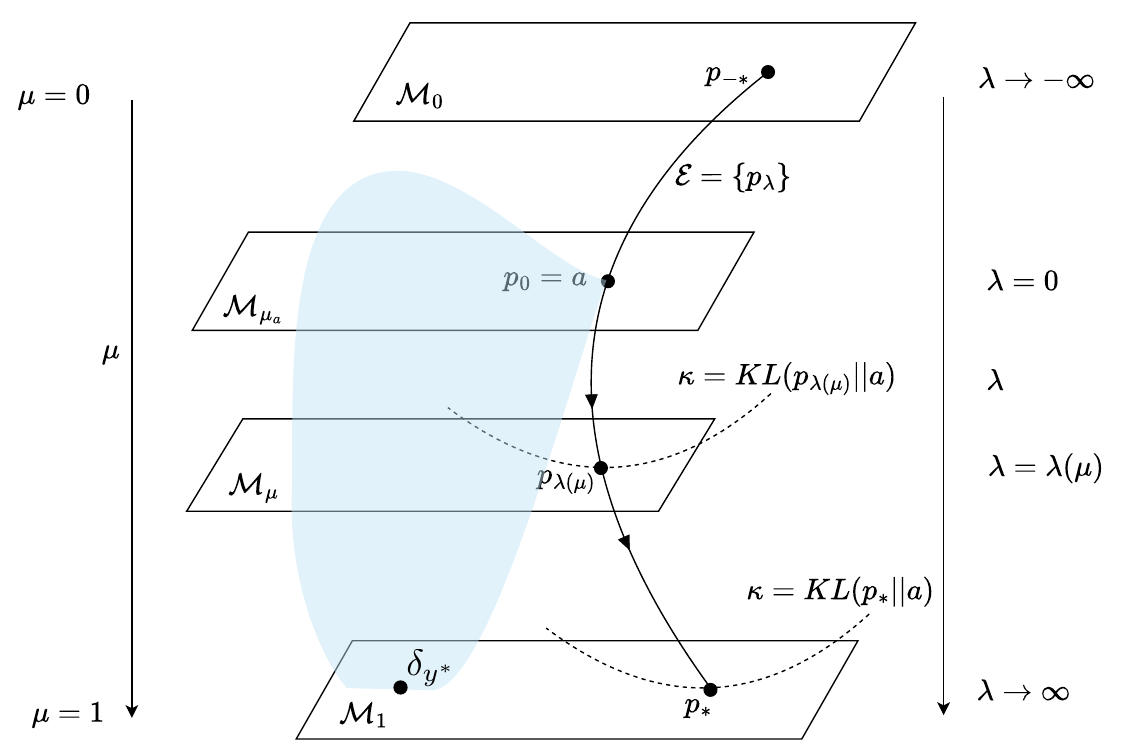}
    \caption{The information-geometry picture of Fig.~\ref{fig:binary-info-geom}, augmented with the parametric family $\Pi_\Theta$ (shaded region).
    The family $\Pi_\Theta$ contains $\aa$ but does not reach $\pp$.
    A near-Dirac policy $\delta_{y^*}$ on a single valid output lies on $\mathcal{M}_1$ (the fully valid slice) and is reachable by $\Pi_\Theta$, whereas $\pp$ lies on the same slice but is not reachable.
    Under large $\lambda$, the optimizer is pushed toward the high-validity boundary of $\Pi_\Theta$---where near-Dirac policies live---rather than toward $\pp$.}
    \label{fig:Pi-Theta-mode-collapse}
\end{figure}

Figure~\ref{fig:Pi-Theta-mode-collapse} illustrates this geometry. Both $\delta_{y^*}$ and $\pp$ lie on the fully valid slice $\mathcal{M}_1$, but only $\delta_{y^*}$ is accessible to $\Pi_\Theta$. As $\lambda$ increases, the target $\pl$ moves toward $\mathcal{M}_1$ along the exponential family curve, exiting $\Pi_\Theta$; the optimizer, unable to follow, is pushed toward the reachable end of $\mathcal{M}_1$---the near-Dirac policies.

We emphasize that this argument combines the exact decomposition~\eqref{eq:KL-difference-candidates} with structural observations about autoregressive families (the ease of memorization, the difficulty of matching $\pp$). The conclusion---that small $\beta$ drives mode collapse under misspecification---is a qualitative prediction grounded in these observations, not a formal guarantee. The following subsection illustrates it concretely on a toy autoregressive model.

\begin{myhighlight}%
KL-control is designed to prevent catastrophic forgetting of the base
model. Yet under misspecification and small $\beta$, it can drive toward
a policy that has maximally forgotten the base model's diversity:
a highly concentrated distribution, collapsing toward ever fewer
valid outputs as $\beta$ decreases. The filtered model $\pp$---the
ideal resolution, achieving perfect validity while preserving the base
model's ranking of valid outputs---sits outside the parametric family
and cannot serve as the optimization target.
\end{myhighlight}

\subsection{Toy illustration: bigram model with binary verifier}
\label{sec:toy-experiment}

We illustrate the mode-collapse mechanism on a minimal autoregressive example that is small enough for exact computation yet exhibits genuine misspecification.

\paragraph{Setup.}
Let $V=\{0,1,2\}$ be a three-token vocabulary and $\Y=V^3$ the set of all sequences of length~$3$, so $|\Y|=27$. The verifier checks whether the first token equals the last: $v(y_1,y_2,y_3)=\mathbf{1}[y_1=y_3]$. This gives $|\Y_1|=9$ valid sequences. The base model $\aa$ is an autoregressive model with full trigram conditionals $\aa(y_1)\,\aa(y_2\mid y_1)\,\aa(y_3\mid y_1,y_2)$, where each conditional is a softmax over logits drawn independently from $\mathcal{N}(0,0.25)$, producing a mildly non-uniform distribution (base validity $\Aone=\aa(\Y_1)\approx 0.35$).

\paragraph{Misspecification via the bigram constraint.}
We consider a \emph{bigram} policy class $\Pi_\Theta$ in which $\pit(y_3\mid y_1,y_2)$ depends only on~$y_2$, not on~$y_1$. Since the verifier requires $y_3=y_1$, the filtered model $\pp=\aa(\cdot\mid\Y_1)$ exhibits a long-range dependency between positions~$1$ and~$3$ that no bigram model can represent. The family $\Pi_\Theta$ therefore does not contain~$\pp$: it is misspecified.

\paragraph{Optimization.}
All quantities---$\Jb(\pit)$, its gradient, and all distributional metrics---are computed exactly by enumerating the $27$ sequences. We optimize $\Jb(\pit)$ by gradient ascent on the logits for a range of $\beta$ values, starting from the base model~$\aa$.

\paragraph{Reference policies.}
To assess what the bigram family is actually capable of, we compute two reference policies. The \emph{forward-KL-optimal bigram policy} $\pihat:=\argmin_{\pit\in\Pi_\Theta}\KL{\pp}{\pit}$ is computed by gradient descent (which finds the global optimum, since this objective is convex in the logits); it has moderate validity ($\muq{\pihat}\approx 0.39$) but low divergence from~$\pp$ ($\KL{\pp}{\pihat}\approx 1.02$). The \emph{TVD-optimal bigram policy} $\argmin_{\pit\in\Pi_\Theta}\TVD{\pit}{\pp}$ is estimated by multi-restart gradient descent with finite-difference gradients (a best-effort estimate, since this objective is non-convex in the logits); it achieves $\TVD{\pit}{\pp}\approx 0.37$. Both policies represent what the bigram family can achieve when optimized toward~$\pp$ under different criteria---neither is ever selected by the KL-control objective.

\begin{figure}[H]
    \centering
    \includegraphics[width=0.7\linewidth]{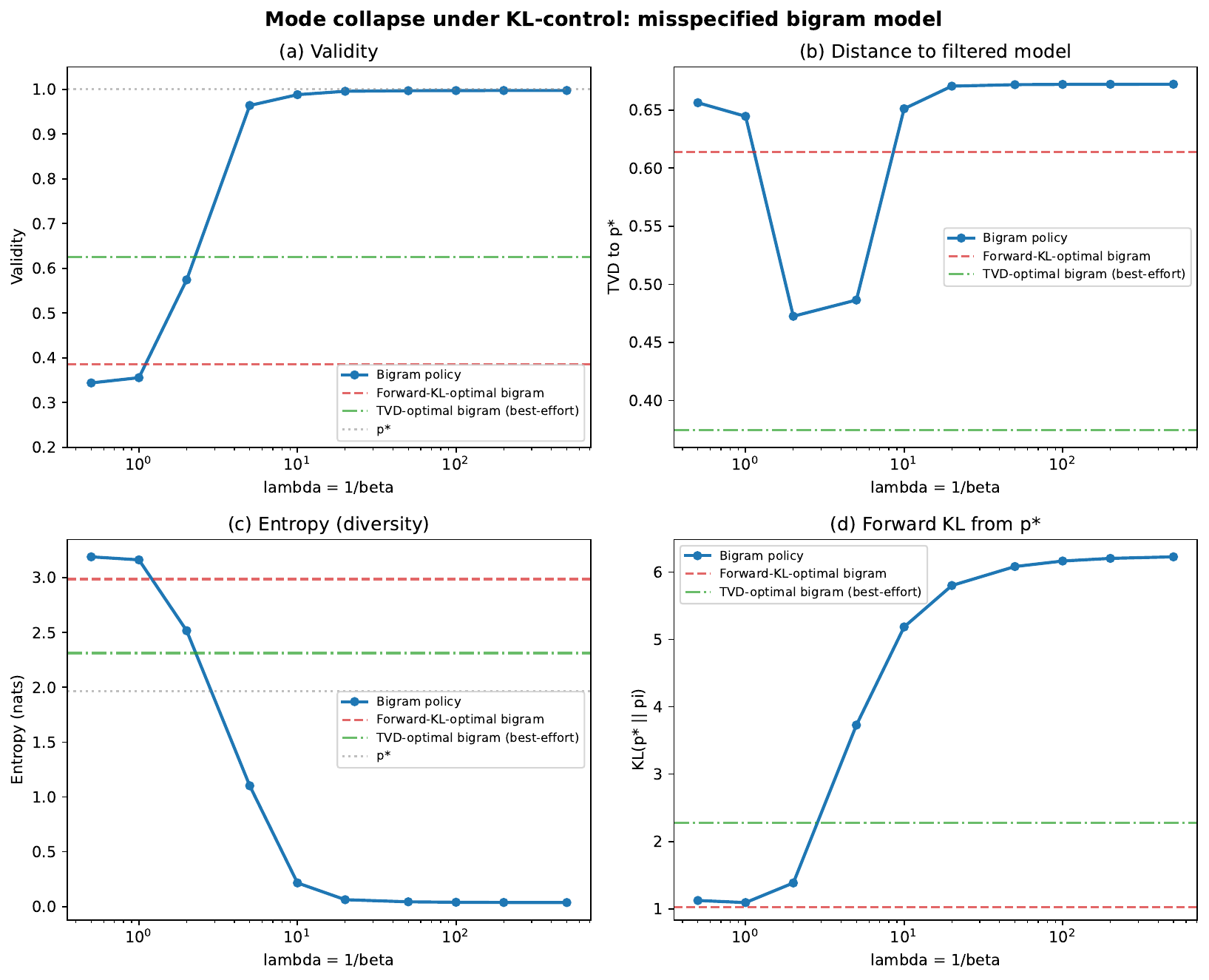}
    \caption{Mode collapse in a misspecified bigram model ($V=\{0,1,2\}$, $T=3$, verifier: $y_1=y_3$). Each panel plots a metric of the optimized bigram policy (solid blue) as a function of $\lambda=1/\beta$. The dashed red line marks the forward-KL-optimal bigram policy $\pihat$ (minimizer of $\KL{\pp}{\pit}$ over $\Pi_\Theta$, found by convex optimization); the dash-dotted green line marks the TVD-optimal bigram policy (best-effort minimizer of $\TVD{\pit}{\pp}$ over $\Pi_\Theta$, found by multi-restart gradient descent); the dotted gray line marks $\pp$ where applicable. As $\lambda$ increases, validity rises to~$1$ but the KL-control policy ends up dramatically farther from~$\pp$---in every metric---than either reference policy, despite both being achievable within the same bigram family $\Pi_\Theta$.}
    \label{fig:toy-mode-collapse}
\end{figure}

\paragraph{Results.}
Figure~\ref{fig:toy-mode-collapse} shows the key metrics as a function of $\lambda=1/\beta$. As $\lambda$ increases:
\begin{itemize}
    \item \emph{Validity} rises to ${\sim}1$, far above $\pihat$ and the TVD-optimal policy.
    \item \emph{TVD to $\pp$} initially decreases as the policy sheds mass from $\Y_0$, but then increases sharply once concentration sets in---the policy overshoots past $\pp$ into near-Dirac territory (see Appendix~\ref{app:tvd-dip} for discussion). At large $\lambda$, the TVD stabilizes near $0.67$---nearly \emph{twice} the ${\sim}0.37$ achievable by the TVD-optimal bigram policy within the same family $\Pi_\Theta$.
    \item \emph{Entropy} collapses from ${\sim}3.1$ (near $\aa$) to ${\sim}0.04$, compared with $\pp$'s entropy of~${\sim}2.0$.
    \item \emph{Forward KL from $\pp$} increases monotonically to ${\sim}6.3$, far exceeding that of $\pihat$ (${\sim}1.0$) and the TVD-optimal policy (${\sim}2.3$).
\end{itemize}
Table~\ref{tab:toy-top-seqs} confirms what these numbers mean concretely. At $\lambda=5$, the policy has concentrated $97\%$ of its mass on three ``diagonal'' sequences $(0,0,0)$, $(1,1,1)$, $(2,2,2)$---the valid sequences that a bigram model can represent by simply copying the previous token. At $\lambda\ge 10$, the mass has further concentrated onto a single sequence, $(0,0,0)$. The remaining six valid sequences, which $\pp$ assigns substantial probability to (together ${\sim}49\%$), are effectively abandoned.

The key observation is that this failure is not an inherent limitation of the bigram family: both reference policies achieve far better approximations of $\pp$ within $\Pi_\Theta$. The KL-control objective, by rewarding validity above all else at small $\beta$, actively drives the optimizer to a policy that is dramatically worse---by every measure of proximity to $\pp$---than what the parametric family is capable of. This is the mode-collapse mechanism of Section~\ref{sec:mode-collapse} in action.

Results are robust across different random base models; see Appendix~\ref{app:toy-details} for a multi-seed analysis.

\begin{table}[H]
\centering
\small
\renewcommand{\arraystretch}{1.1}
\begin{tabular}{@{}ccl@{}}
\toprule
$\lambda$ & Entropy & Top sequences (probability) \\
\midrule
$2$  & $2.52$ & $(0,0,0)$: 37\%;\; spread across ${\sim}19$ sequences \\
$5$  & $1.10$ & $(0,0,0)$: 63\%;\; $(1,1,1)$: 20\%;\; $(2,2,2)$: 14\% \\
$10$ & $0.22$ & $(0,0,0)$: 97\%;\; $(0,2,0)$: 1.7\%;\; $(0,1,0)$: 0.5\% \\
$20$ & $0.06$ & $(0,0,0)$: 99.3\% \\
$100$ & $0.04$ & $(0,0,0)$: 99.6\% \\
\bottomrule
\end{tabular}
\caption{Top sequences and their probabilities for the optimized bigram policy at selected values of $\lambda=1/\beta$. At intermediate $\lambda$, the policy concentrates on the three ``diagonal'' valid sequences (where all tokens coincide); for large $\lambda$, it collapses further onto a single sequence.}
\label{tab:toy-top-seqs}
\end{table}

\begin{rem}[Well-specified case and the optimization gap]\label{rem:well-specified}
When the policy class is a full trigram model (no bigram restriction), the family $\Pi_\Theta$ contains $\pp$ and the misspecification gap vanishes. Exact gradient descent on $\Jb(\pit)$ initialized at $\aa$ nonetheless exhibits collapse for large $\lambda$: the steep validity gradient early in optimization funnels the trajectory toward a near-Dirac basin before it can reach the neighborhood of $\pp$. This is an \emph{optimization gap}, not a misspecification gap. Warm-starting from a moderate-$\lambda$ solution substantially reduces the collapse, confirming its path-dependent nature. We do not pursue this further here, but note that it suggests the difficulties of small-$\beta$ optimization extend beyond the misspecified setting.
\end{rem}

% ============================================================
\section{Discussion}
\label{sec:discussion}
% ============================================================

\subsection{The divergence choice problem}
\label{sec:divergence-choice}

A recurring theme of this paper is that $\pp$ is the natural target for binary-reward RL, but the standard KL-control objective cannot target it directly.

The obstacle is structural. KL-control works by minimizing the reverse KL, $\KL{q}{\pb}$, and $\pb$ is full-support, so this is always well-defined. But $\pp$ is \emph{not} full-support: it assigns zero probability to all $y\in\Y_0$. Consequently, $\KL{q}{\pp}=\infty$ for any $q$ that puts mass on $\Y_0$---which includes every policy $\pit$ in a standard autoregressive family. The reverse KL, $\KL{\cdot}{\pp}$, is simply not a viable objective.

This disqualification is specific to the choice of divergence and to the argument position of $\pp$. The \emph{forward} KL, $\KL{\pp}{q}$, is perfectly well-defined whenever $q$ has full support, and autoregressive policies typically do. Minimizing $\KL{\pp}{\pit}$ over $\Pi_\Theta$ is therefore a well-posed optimization problem that targets $\pp$ directly. This was recognized early on by \citet{khalifa_distributional_2021}, who optimized the forward KL to the filtered model (under the name ``pointwise constraint'') as part of a broader distributional approach, and pursued further by \citet{kim_guaranteed_2025} with explicit focus on $\pp$.

More generally, $\alpha$-divergences \citep{renyi_measures_1961,amari_information_2016} $D_\alpha(\pit\|\pp)$ are finite for appropriate choices of $\alpha$ even when $\pp$ has restricted support, offering a family of interpolating objectives between the forward KL ($\alpha\to 0$) and the reverse KL ($\alpha\to 1$) \citep{kruszewski_whatever_2026}. In a related direction, \citet{go_aligning_2023} advocate f-divergence minimization for aligning language models with a target distribution, though without specific focus on filtered models such as $\pp$.

The practical implications are significant. The mode-collapse mechanism identified in Section~\ref{sec:mode-collapse} is driven by the reverse-KL structure of the standard objective. A forward-KL or $\alpha$-divergence objective that targets $\pp$ directly could in principle avoid this failure mode, since it would reward \emph{coverage} of $\pp$'s support rather than \emph{concentration} on high-validity outputs.

The key point is not merely that alternative divergences are worth exploring, but that the \emph{standard} divergence used in KL-control is provably unable to target $\pp$ at all. This is a structural impossibility, not a matter of degree: no amount of tuning $\beta$ can make the reverse KL to $\pp$ a well-defined objective. The analysis of Section~\ref{sec:mode-collapse} shows that this impossibility has concrete consequences---mode collapse under misspecification is driven precisely by the reverse-KL structure.

Recent empirical work corroborates this picture. Extending the forward-KL approach of \citet{khalifa_distributional_2021} and \citet{kim_guaranteed_2025} to the full $\alpha$-divergence family, \citet{kruszewski_whatever_2026} introduce $\alpha$-DPG, which minimizes $D_\alpha(\pit\|\pp)$ for varying $\alpha$, with the theoretical underpinning provided by their Support Decomposition theorem (Appendix~H, Theorem~5). They find that lower $\alpha$ (closer to the forward KL) yields better diversity, tracing a Pareto frontier between precision and coverage. Independently, \citet{li_choice_2025} identify the reverse KL as a primary cause of diversity collapse and show that forward-KL or Jensen--Shannon alternatives significantly improve both pass@1 and pass@$k$. Of course, optimizing $\KL{\pp}{\pit}$ presents its own challenges---it requires sampling from $\pp$ or a good approximation---but the theoretical case for moving beyond the reverse KL is clear.

\subsection{Summary and broader perspective}
\label{sec:summary}

We have argued that binary rewards create a specific, identifiable pathology for KL-controlled reinforcement learning:

\begin{enumerate}
    \item The \reinforce objective is degenerate: the set of optimal distributions $\PP_1$ is infinite, with no distinguished element (Section~\ref{sec:binary}).
    \item KL-control resolves this degeneracy by selecting the filtered model $\pp$ in the $\beta\to 0$ limit (Section~\ref{sec:convergence}). The convergence is one-sided: forward KL converges, reverse KL stays infinite.
    \item The hyperparameter $\beta$ is opaque; the target validity rate $\mu$ is more interpretable and invariant across base models of different strength (Section~\ref{sec:beta-vs-mu}).
    \item Under misspecification, optimization pressure for small $\beta$ tends toward mode collapse on a small number of valid outputs, not toward $\pp$ (Section~\ref{sec:mode-collapse}), as confirmed by a toy autoregressive experiment (Section~\ref{sec:toy-experiment}).
\end{enumerate}

The distribution-matching perspective---viewing KL-control as projecting onto a target $\pb$ rather than as reward maximization with a penalty---is essential for understanding both what the method achieves (resolving the degeneracy, in the ideal case) and where it fails (mode collapse, under misspecification).

Several directions remain open. The present paper establishes a structural
reason why the standard reverse-KL objective cannot target $\pp$ directly,
and recent work \citep{kruszewski_whatever_2026,li_choice_2025} shows
empirically that alternative divergences improve diversity. What is missing
is a precise theoretical account of the divergence choice problem under
misspecification: which divergences, or more generally which comparisons
between distributions, achieve which guarantees when $\pp$ is unreachable?
The relationship between the misspecification gap and the architecture of
the policy class also remains largely unexplored. More broadly, it would
be interesting to extend the analysis to continuous rewards, multi-level
structures, or conditional distributions $p(y\mid x)$.

\section*{Acknowledgments}
The author thanks Germán Kruszewski for insightful comments on an earlier version of this paper.

%%%%%%%%%%%%%%%%%%%%%%%%%%%%%%%%%%%%%%%%%%%%%%%%%%%%%%%%%%%%%%%%%%
\section*{AI Disclosure}
%%%%%%%%%%%%%%%%%%%%%%%%%%%%%%%%%%%%%%%%%%%%%%%%%%%%%%%%%%%%%%%%%%

The author used Claude (Anthropic, \texttt{claude.ai}) and ChatGPT (OpenAI) 
during the preparation of this manuscript for assistance with exposition, 
structuring arguments, reviewing text and proof drafts, and developing and 
debugging the code for the toy experiment. The author reviewed and edited 
all AI-assisted content, and takes full responsibility for the correctness of this paper.

% ============================================================
% BIBLIOGRAPHY
% ============================================================

% ============================================================
% APPENDICES
% ============================================================

\newpage
% ============================================================
% APPENDICES
% ============================================================

\newpage
\appendix

\section{Proof of Proposition~\ref{thm:pstar-I-proj}: $\pp$ is the I-projection of $\aa$ onto $\PP_1$}
\label{app:p*-is-I-proj-proof}

\begin{proof}
For any $q\in\PP_1$ (i.e., $\supp(q)\subseteq\Y_1$), every $y$ in the support of $q$ satisfies $v(y)=1$, so $\aa(y) = \Aone\cdot\pp(y)$. Therefore
\begin{align*}
    \KL{q}{\aa}
    &= \sum_{y\in\Y_1} q(y)\log\frac{q(y)}{\aa(y)}
     = \sum_{y\in\Y_1} q(y)\log\frac{q(y)}{\Aone\,\pp(y)} \\
    &= \sum_{y\in\Y_1} q(y)\log\frac{q(y)}{\pp(y)} + \sum_{y\in\Y_1} q(y)\log\frac{1}{\Aone} \\
    &= \KL{q}{\pp} + \log\frac{1}{\Aone}.
\end{align*}
Since $\log(1/\Aone)$ is a positive constant independent of $q$, the minimum of $\KL{q}{\aa}$ over $\PP_1$ is attained when $\KL{q}{\pp}=0$, i.e., when $q=\pp$. The minimum value is $\KL{\pp}{\aa}=\log(1/\Aone)=-\log\Aone$.
\end{proof}

\section{Proof of Theorem~\ref{thm:main}: Convergence of $\pb$ to $\pp$}
\label{app:proof:convergence}

We work throughout in the binary case $r(y)=v(y)\in\{0,1\}$, using the parametrization $\lambda=1/\beta$, so that $\beta\to 0^+$ corresponds to $\lambda\to+\infty$. Then $\Zl=\Azero+\Aone\,e^\lambda$ and, for $y\in\Y$,
\[
\pl(y) = \frac{\aa(y)\,e^{\lambda\,v(y)}}{\Azero+\Aone\,e^\lambda}
= \begin{cases}
    \dfrac{\aa(y)\,e^\lambda}{\Azero+\Aone\,e^\lambda} & \text{if } y\in\Y_1, \\[8pt]
    \dfrac{\aa(y)}{\Azero+\Aone\,e^\lambda} & \text{if } y\in\Y_0.
\end{cases}
\]

\begin{proof}[Proof of \ref{thm:main:a} (pointwise convergence)]
For $y\in\Y_1$:
\[
\pl(y) = \frac{\aa(y)\,e^\lambda}{\Azero+\Aone\,e^\lambda}
= \frac{\aa(y)}{\Azero\,e^{-\lambda}+\Aone}
\;\xrightarrow{\lambda\to+\infty}\;
\frac{\aa(y)}{\Aone} = \pp(y).
\]
For $y\in\Y_0$:
\[
\pl(y) = \frac{\aa(y)}{\Azero+\Aone\,e^\lambda}
\;\xrightarrow{\lambda\to+\infty}\; 0 = \pp(y). \qedhere
\]
\end{proof}

\begin{proof}[Proof of \ref{thm:main:b} (total variation)]
We compute the total variation directly. On $\Y_0$, $\pp(y)=0$, so $|\pp(y)-\pl(y)|=\pl(y)$ and the contribution is
$\sum_{y\in\Y_0}\pl(y) = \pl(\Y_0)$.
On $\Y_1$, both $\pp(y)$ and $\pl(y)$ are positive multiples of $\aa(y)$, and a short calculation gives
\[
\pp(y)-\pl(y) = \frac{\aa(y)}{\Aone}-\frac{\aa(y)\,e^\lambda}{\Zl}
= \frac{\aa(y)\,\Azero}{\Aone\,\Zl}\;\ge\;0,
\]
so the contribution is $\sum_{y\in\Y_1}(\pp(y)-\pl(y)) = 1 - \pl(\Y_1) = \pl(\Y_0)$. Combining,
\[
\TVD{\pp}{\pl}
= \tfrac{1}{2}\bigl(\pl(\Y_0)+\pl(\Y_0)\bigr)
= \pl(\Y_0)
= \frac{\Azero}{\Azero+\Aone\,e^\lambda}
\;\xrightarrow{\lambda\to+\infty}\; 0.\qedhere
\]
\end{proof}

\begin{proof}[Proof of \ref{thm:main:c} (forward KL)]
\begin{align*}
\KL{\pp}{\pl}
&= \sum_{y\in\Y_1}\pp(y)\log\frac{\pp(y)}{\pl(y)}
= \sum_{y\in\Y_1}\pp(y)\log\frac{\aa(y)/\Aone}{\aa(y)\,e^\lambda/\Zl} \\
&= \sum_{y\in\Y_1}\pp(y)\log\frac{\Zl}{\Aone\,e^\lambda}
= \log\frac{\Zl}{\Aone\,e^\lambda}.
\end{align*}
Now $\Zl/(\Aone\,e^\lambda) = (\Azero+\Aone\,e^\lambda)/(\Aone\,e^\lambda) = 1 + \Azero\,e^{-\lambda}/\Aone$. Therefore
\[
\KL{\pp}{\pl} = \log\bigl(1 + \tfrac{\Azero}{\Aone}\,e^{-\lambda}\bigr) \;\xrightarrow{\lambda\to+\infty}\; 0. \qedhere
\]
\end{proof}

\begin{proof}[Proof of \ref{thm:main:d} (reverse KL is infinite)]
For every $\beta>0$ (equivalently, every finite $\lambda$), $\pl$ has full support: $\pl(y)>0$ for all $y\in\Y$. In particular, $\pl(y)>0$ for every $y\in\Y_0$. But $\pp(y)=0$ for $y\in\Y_0$, so
\[
\KL{\pl}{\pp} = \sum_{y\in\Y}\pl(y)\log\frac{\pl(y)}{\pp(y)} = +\infty,
\]
since the sum includes terms $\pl(y)\log(\pl(y)/0)=+\infty$ for $y\in\Y_0$ with $\pl(y)>0$.
\end{proof}

\section{Proof of the KL Difference Identity}
\label{app:kl-identity-proof}

We prove identity~\eqref{eq:KL-identity}: for any $q\in\PP$ and any $\lambda_1,\lambda_2\in\R$,
\[
\KL{q}{\pl[\lambda_2]} - \KL{q}{\pl[\lambda_1]}
= A(\lambda_2) - A(\lambda_1) + \muq{q}\,(\lambda_1 - \lambda_2),
\]
where $\muq{q}:=\E_q[r]$ and $A(\lambda):=\log\Zl$.

\begin{proof}
By definition of $\pl$, $\log\pl(y) = \log\aa(y) + \lambda\,r(y) - A(\lambda)$. Hence
\begin{align*}
\KL{q}{\pl[\lambda_2]} - \KL{q}{\pl[\lambda_1]}
&= \E_q\!\left[\log\frac{\pl[\lambda_1](y)}{\pl[\lambda_2](y)}\right] \\
&= \E_q\bigl[(\lambda_1-\lambda_2)\,r(y) - A(\lambda_1) + A(\lambda_2)\bigr] \\
&= A(\lambda_2) - A(\lambda_1) + \muq{q}\,(\lambda_1 - \lambda_2). \qedhere
\end{align*}
\end{proof}

Setting $\lambda_1=0$ (so $\pl[\lambda_1]=\aa$ and $A(\lambda_1)=0$), $\lambda_2=\lambda=1/\beta$, and rearranging recovers Proposition~\ref{thm:korbak:a}: $\Jb(q) = \beta[\log\Zb - \KL{q}{\pb}]$. The general multi-dimensional version of this identity, and a systematic derivation of exponential family theory from it, is developed in \citet{dymetman_exponential_2026}.

\section{Information Geometry of KL-Control: General Bounded Case}
\label{app:general-info-geom}

This appendix develops the full information-geometric picture for arbitrary bounded rewards, of which the binary case treated in the main text is a specialization. The results here are standard in the theory of exponential families; \citet{dymetman_exponential_2026} offers a self-contained and unified treatment, deriving the Pythagorean theorem, I-projection characterizations, Legendre duality, and KL-regularized optimization as short consequences of the KL difference identity~\eqref{eq:KL-identity}, in the general multi-dimensional setting.

\paragraph{What generalizes and what does not.}
The binary case of Section~\ref{sec:binary-formulas} is special in that the reward takes only two values, so $\Zl=\Azero+\Aone\,e^\lambda$ collapses to a two-term sum and the moment map $\mu(\lambda)$, its inverse $\lambda(\mu)$, and the divergence cost $\kappa(\mu)$ all admit closed elementary expressions in terms of $\Aone$ and $\Azero$. For a general bounded reward $r$, no such simplification is available: $A(\lambda)=\log\sum_y\aa(y)\,e^{\lambda r(y)}$ is in general only implicitly defined, and $\mu$, $\lambda$, $\kappa$ are mutually related by smooth bijections without elementary closed forms. What \emph{does} survive is the geometric picture---the exponential family curve threading through the moment slices, the I-projection characterization, the Legendre duality, and the tangency of KL sublevel sets---which is what this appendix develops.

\subsection{Moment map and exponential family}

Let $r:\Y\to\R$ be bounded and non-constant, with $m:=\inf_y r(y)$ and $M:=\sup_y r(y)$ finite but not necessarily attained. Consider the exponential family $\EE=\{\pl\}_{\lambda\in\R}$ defined by~\eqref{eq:1-dim-exp-fam}, with log-partition function $A(\lambda)=\log\Zl$ and moment map $\mu(\lambda):=\E_{\pl}[r]$.

\begin{prop}[Moment map for bounded rewards]\label{thm:moment-map-bounded}
$A$ is finite and infinitely differentiable on $\R$, with
$A'(\lambda)=\mu(\lambda)$ and $A''(\lambda)=\Var_{\pl}(r)$.
Since $r$ is non-constant and $\pl$ has full support, $\Var_{\pl}(r)>0$, so $A$ is strictly convex and $\lambda\mapsto\mu(\lambda)$ is strictly increasing, with
$\lim_{\lambda\to-\infty}\mu(\lambda)=m$ and $\lim_{\lambda\to+\infty}\mu(\lambda)=M$.
Thus $\lambda\mapsto\mu(\lambda)$ is a bijection from $\R$ onto $(m,M)$, with inverse $\mu\mapsto\lambda(\mu)$.
\end{prop}

\noindent The identification $A'=\mu$, the strict convexity of $A$, and the characterization $A''=\Var_{\pl}(r)$ are standard; see, e.g., \citet[Proposition~3.1]{wainwright_graphical_2008}.

\subsection{Moment slices and I-projections}

For $\mu\in[m,M]$, define the moment slice $\mathcal{M}_\mu:=\{q\in\PP:\E_q[r]=\mu\}$.
Since $\mu(\lambda)$ is a bijection onto $(m,M)$, the exponential family curve intersects each interior slice $\mathcal{M}_\mu$ ($\mu\in(m,M)$) at exactly one point, namely $\pl[\lambda(\mu)]$.

\begin{prop}[I-projection onto moment slices]
For $\mu\in(m,M)$, $\pl[\lambda(\mu)] = \argmin_{q\in\mathcal{M}_\mu}\KL{q}{\aa}$.
\end{prop}

\begin{proof}
Apply identity~\eqref{eq:KL-identity} with $\lambda_1=0$, $\lambda_2=\lambda(\mu)$, and $q\in\mathcal{M}_\mu$ (so $\muq{q}=\mu$):
\[
\KL{q}{\pl[\lambda(\mu)]} - \KL{q}{\aa} = A(\lambda(\mu)) - \mu\,\lambda(\mu).
\]
The right-hand side is independent of $q$, so minimizing $\KL{q}{\aa}$ over $\mathcal{M}_\mu$ is equivalent to minimizing $\KL{q}{\pl[\lambda(\mu)]}$ over $\mathcal{M}_\mu$. Since $\pl[\lambda(\mu)]\in\mathcal{M}_\mu$, the latter minimum is~$0$, attained uniquely at $q=\pl[\lambda(\mu)]$.
\end{proof}

Thus the exponential family curve can be viewed as the locus of I-projections of $\aa$ onto the stacked moment slices.

\subsection{Legendre dual and the $\lambda\leftrightarrow\mu\leftrightarrow\kappa$ bijection}

The convex conjugate of $A$ is
$A^*(\mu):=\sup_{\lambda\in\R}\{\lambda\mu-A(\lambda)\}$.

\begin{prop}[Legendre dual]\label{thm:legendre-dual-bounded-app}
\begin{enumerate}
\item $A^*(\mu)<\infty$ for $\mu\in(m,M)$ and $A^*(\mu)=+\infty$ for $\mu\notin[m,M]$. At the endpoints, finiteness holds iff the bound is attained.

\item For $\mu\in(m,M)$, the supremum is attained at $\lambda(\mu)$ and $A^*(\mu)=\lambda(\mu)\mu-A(\lambda(\mu))$.

\item $A^*(\mu)=\KL{\pl[\lambda(\mu)]}{\aa}=\min_{q:\,\E_q[r]=\mu}\KL{q}{\aa}$.
\end{enumerate}
\end{prop}

\noindent Parts (1)--(2) follow from the strict convexity of $A$ and the boundary behavior of $\mu(\lambda)$. Part~(3) combines the Legendre identity $A^*(\mu)=\lambda\mu-A(\lambda)$ with the I-projection characterization above; see \citet[Chapter~3]{wainwright_graphical_2008}.

Writing $\kappa(\mu):=A^*(\mu)$, we obtain a strictly increasing bijection $\mu\leftrightarrow\kappa$ over $[\muq{\aa},M)$, which, combined with $\lambda\leftrightarrow\mu$, gives the three-way bijection $\lambda\leftrightarrow\mu\leftrightarrow\kappa$ over $[0,+\infty)\leftrightarrow[\muq{\aa},M)\leftrightarrow[0,\kappa(M))$.

\subsection{KL sublevel sets and tangency}

For $\kappa\ge 0$, the sublevel set $\PP^\kappa:=\{q\in\PP:\KL{q}{\aa}\le\kappa\}$ is convex. The function $\kappa(\mu)$ is strictly convex and strictly increasing on $[\muq{\aa},M)$, and the boundary of $\PP^{\kappa(\mu)}$ meets the moment slice $\mathcal{M}_\mu$ at exactly one point, $\pl[\lambda(\mu)]$. Geometrically, this tangency reflects the supporting-hyperplane characterization of the Legendre dual: $\lambda(\mu)$ is the slope of $\kappa$ at $\mu$, and the linear functional $q\mapsto\E_q[r]$ defining $\mathcal{M}_\mu$ is the corresponding supporting hyperplane to $\PP^{\kappa(\mu)}$ at $\pl[\lambda(\mu)]$. The trade-off is then immediate: any distribution with expected reward exceeding $\mu$ must have KL divergence to $\aa$ exceeding $\kappa(\mu)$. See Fig.~\ref{fig:iprojs-general} for an illustration.

\subsection{Attained bounds}

\paragraph{Attained upper bound.}
If $\Y_M:=\{y:r(y)=M\}\neq\varnothing$, then $\pp:=\aa(\cdot\mid\Y_M)$ satisfies:
\begin{align}
    &\pl\to\pp \qquad (\lambda\to+\infty),\\
    &\KL{\pp}{\pl}\to 0 \qquad (\lambda\to+\infty),\\
    &\KL{\pl}{\pp}=+\infty \qquad (\forall\,\lambda\in\R),\\
    &\KL{\pl}{\aa}\to A^*(M) = \KL{\pp}{\aa} = -\log\aa(\Y_M).
\end{align}
The proofs are analogous to Appendix~\ref{app:proof:convergence}: pointwise convergence follows from the explicit form of $\pl$, forward KL convergence from the asymptotics of $A(\lambda)$, and the reverse KL infinity from the support mismatch between $\pl$ (full support) and $\pp$ (supported on $\Y_M$).

\paragraph{Attained lower bound.}
Symmetrically, if $\Y_m:=\{y:r(y)=m\}\neq\varnothing$, then $p_{-*}:=\aa(\cdot\mid\Y_m)$ satisfies the analogous statements with $\lambda\to-\infty$ in place of $\lambda\to+\infty$ and the roles of $\pl$ and $p_{-*}$ exchanged in the reverse-KL statement:
\begin{align}
    &\pl\to p_{-*} \qquad (\lambda\to-\infty),\\
    &\KL{p_{-*}}{\pl}\to 0 \qquad (\lambda\to-\infty),\\
    &\KL{\pl}{p_{-*}}=+\infty \qquad (\forall\,\lambda\in\R),\\
    &\KL{\pl}{\aa}\to A^*(m) = \KL{p_{-*}}{\aa} = -\log\aa(\Y_m).
\end{align}
In the binary case both bounds are attained, and we recover Theorem~\ref{thm:main} (upper bound, $\pp=\aa(\cdot\mid\Y_1)$) together with its lower-bound counterpart involving $p_{-*}=\aa(\cdot\mid\Y_0)$.

\begin{figure}[H]
    \centering
    \includegraphics[width=0.75\linewidth]{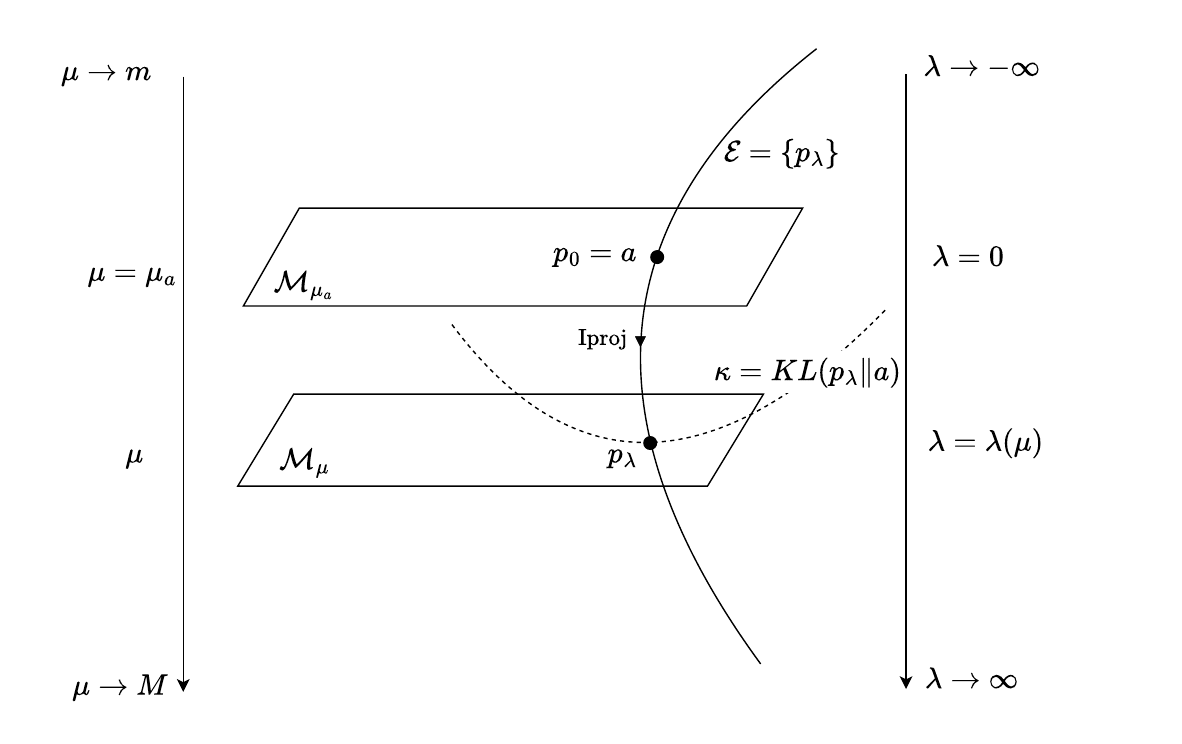}
    \caption{Information geometry of KL-control for a general bounded reward (bounds not necessarily attained).
The moment slices $\mathcal{M}_\mu$ are stacked vertically; the exponential family $\EE=\{\pl\}$ intersects each slice at the I-projection of $\aa$. The dashed curve labeled $\kappa$ shows the KL sublevel set $\PP^\kappa$ tangent to $\mathcal{M}_\mu$ at $\pl[\lambda(\mu)]$.
Fig.~\ref{fig:binary-info-geom} shows the same picture for the binary case, where both reward bounds are attained: the boundary slices $\mathcal{M}_0$ and $\mathcal{M}_1$ then contain the endpoint distributions $p_{-*}=\aa(\cdot\mid\Y_0)$ and $\pp=\aa(\cdot\mid\Y_1)$, which appear as additional limit points of the exponential family curve.}
    \label{fig:iprojs-general}
\end{figure}

\section{Toy Experiment: Additional Details}
\label{app:toy-details}

This appendix provides additional details and robustness checks for the toy experiment of Section~\ref{sec:toy-experiment}.

\subsection{Experimental setup}

The sample space is $\Y=V^T$ with $V=\{0,1,2\}$ and $T=3$, giving $|\Y|=27$. The verifier is $v(y_1,y_2,y_3)=\mathbf{1}[y_1=y_3]$, yielding $|\Y_1|=9$ valid sequences. The base model $\aa$ is a full trigram autoregressive model: $\aa(y)=\aa(y_1)\,\aa(y_2\mid y_1)\,\aa(y_3\mid y_1,y_2)$, where each conditional is a softmax over randomly generated logits (drawn i.i.d.\ from $\mathcal{N}(0,0.25)$).

The bigram policy class $\Pi_\Theta$ parametrizes $\pit(y)=\pit(y_1)\,\pit(y_2\mid y_1)\,\pit(y_3\mid y_2)$, where $\pit(y_3\mid y_2)$ does not depend on $y_1$. Each conditional is a softmax over trainable logits, giving $3+9+9=21$ real parameters. Since $|\Y|=27$, all expectations, KL divergences, and gradients are computed by exact enumeration (no sampling). Optimization uses gradient ascent on $\Jb(\pit)$ with learning rate $0.1$ and $8{,}000$ steps, initialized at the base model. Analytical gradients were verified against two-sided finite differences (relative error $<10^{-8}$ on all parameters).

The forward-KL-optimal bigram policy $\pihat$ is computed by minimizing $\KL{\pp}{\pit}$ over $\Pi_\Theta$ via gradient descent (learning rate $0.05$, $15{,}000$ steps). This minimization is convex in the logits: writing $-\log\pi_\theta(y)$ as a sum of negative log-softmax terms (one per conditional, each depending on a disjoint block of logits), and noting that each negative log-softmax, $-\log\mathrm{softmax}_i(z) = \mathrm{LSE}(z) - z_i$, is convex in its argument $z$, the cross-entropy $\KL{\pp}{\pi_\theta} = \mathrm{const} - \sum_y\pp(y)\log\pi_\theta(y)$ is convex in the full parameter vector as a sum of convex functions. Gradient descent therefore finds the global optimum.

The TVD-optimal bigram policy is estimated by minimizing $\TVD{\pit}{\pp}$ over $\Pi_\Theta$ via multi-restart gradient descent with finite-difference gradients ($200$ restarts, $5{,}000$ steps each, decaying learning rate). Unlike the forward-KL case, this problem is non-convex in the logits (TVD is convex on the simplex, but the bigram subset is not), so global optimality is not guaranteed; we treat the result as a best-effort lower bound on the achievable TVD within $\Pi_\Theta$.

\subsection{Robustness across base models}

\begin{figure}[H]
    \centering
    \includegraphics[width=0.75\linewidth]{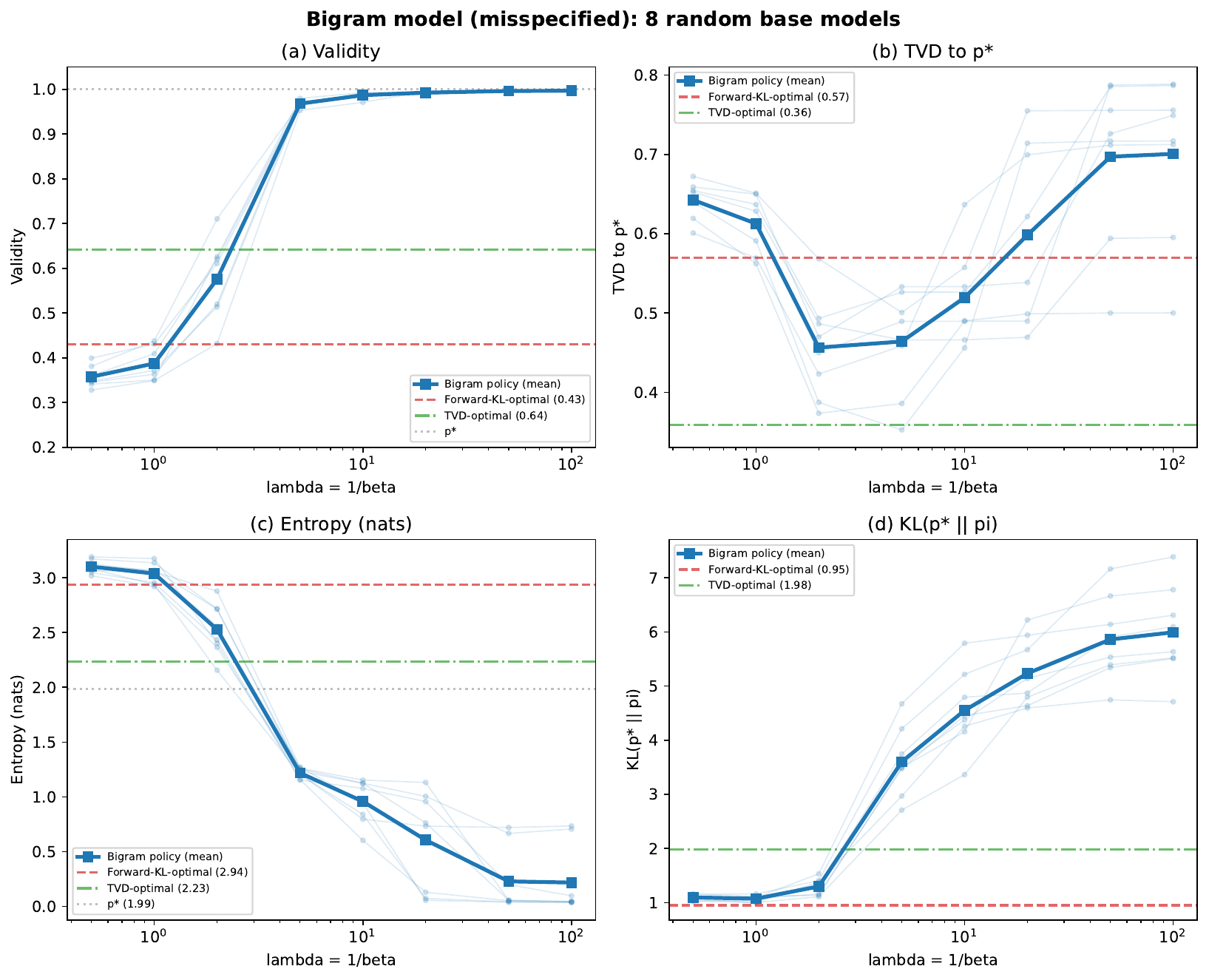}
    \caption{Robustness of mode collapse across $8$ random base models. Light curves show individual seeds; the heavy curve shows the mean. The dashed red line marks the mean forward-KL-optimal bigram policy $\pihat$; the dash-dotted green line marks the mean TVD-optimal bigram policy (best-effort). Both reference lines are \emph{means} across seeds, so individual seed trajectories need not lie above them uniformly---a seed whose own per-seed baseline happens to be low may cross the mean line. The overall pattern---rising validity, rising forward KL from $\pp$, collapsing entropy, and TVD to $\pp$ ending far above the reference policies---is consistent across all seeds. The TVD dip visible in some seeds at intermediate $\lambda$ reflects the transient effect discussed in Appendix~\ref{app:tvd-dip}.}
    \label{fig:toy-multiseed}
\end{figure}

Figure~\ref{fig:toy-multiseed} repeats the experiment of Section~\ref{sec:toy-experiment} for $8$ independently generated random base models (i.e., different random logits for $\aa$). The base-model validity $\Aone=\aa(\Y_1)$ ranges from $0.25$ to $0.42$ across seeds. The mode-collapse pattern is consistent across all seeds: validity approaches $1$, entropy collapses toward $0$, and the forward KL from $\pp$ increases, confirming that the phenomenon is not an artifact of a particular choice of base model.

At $\lambda=50$ ($\beta=0.02$), the mean validity across seeds is $0.996\pm 0.001$, the mean TVD to $\pp$ is $0.70\pm 0.09$, the mean entropy is $0.23\pm 0.27$, and the mean $\KL{\pp}{\pit}$ is $5.86\pm 0.73$. By contrast, the forward-KL-optimal bigram policy $\pihat$ achieves a mean validity of only $0.43$ but a mean $\KL{\pp}{\pihat}$ of $0.95$, and the TVD-optimal policy achieves a mean $\TVD{\pit}{\pp}$ of $0.36$---both substantially closer to $\pp$ than the KL-control solution at any large $\lambda$. The KL-control objective thus drives the optimizer to a policy that is worse than either reference by every measure of proximity to $\pp$, despite both references being achievable within the same bigram family.

\subsection{The TVD dip at intermediate $\lambda$}
\label{app:tvd-dip}

The TVD to $\pp$ exhibits a transient dip around $\lambda\approx 2$--$5$ (visible in both Figures~\ref{fig:toy-mode-collapse} and~\ref{fig:toy-multiseed}), while the forward KL from $\pp$ increases monotonically. This discrepancy reflects the different sensitivities of the two metrics.

At intermediate $\lambda$, the policy is transferring mass from $\Y_0$ to $\Y_1$, which mechanically reduces the TVD to $\pp$ (since $\pp$ is supported on $\Y_1$). However, the mass arriving on $\Y_1$ is not distributed like $\pp$---it is already concentrating on a few valid sequences. The TVD does not yet penalize this heavily because the dominant contribution (mass on $\Y_0$) is decreasing faster than the secondary contribution (misallocation within $\Y_1$) is increasing. Once the mass transfer is essentially complete ($\lambda\gtrsim 5$), the misallocation within $\Y_1$ dominates and the TVD rises sharply.

The forward KL, $\KL{\pp}{\pit}$, detects the misallocation earlier because it is more sensitive to regions where $\pp(y)$ is substantial but $\pit(y)$ is small: each such region contributes a term $\pp(y)\log(\pp(y)/\pit(y))$, which grows rapidly as $\pit(y)$ decreases. As a result, the forward KL increases monotonically throughout the sweep, even during the phase when TVD temporarily improves.

\end{document}